\documentclass[nohyperref]{article}

\usepackage{microtype}
\usepackage{graphicx,enumitem}
\usepackage[]{caption,subcaption}
\usepackage{booktabs} 
\usepackage{hyperref}

\usepackage[accepted]{icml2023}

\usepackage{amsmath}
\usepackage{amssymb}
\usepackage{mathtools}
\usepackage{amsthm}
\usepackage{bm, colonequals}

\graphicspath{{plots/}}

\theoremstyle{plain}
\newtheorem{theorem}{Theorem}[section]
\newtheorem{proposition}[theorem]{Proposition}

\newtheorem{claim}[theorem]{Claim}
\theoremstyle{definition}

\theoremstyle{remark}

\newcommand{\bh}{\mathbf{h}}
\newcommand{\be}{\mathbf{e}}

\newcommand{\ba}{\mathbf{a}}

\newcommand{\bc}{\mathbf{c}}
\newcommand{\bb}{\mathbf{b}}

\newcommand{\bx}{\mathbf{x}}
\newcommand{\by}{\mathbf{y}}

\newcommand{\bz}{\mathbf{z}}

\newcommand{\bP}{\mathbf{P}}
\newcommand{\bA}{\mathbf{A}}

\newcommand{\bW}{\mathbf{W}}

\newcommand{\bG}{\mathbf{G}}

\newcommand{\bL}{\mathbf{L}}

\newcommand{\bI}{\mathbf{I}}

\newcommand{\bV}{\mathbf{V}}
\newcommand{\bK}{\mathbf{K}}
\newcommand{\bX}{\mathbf{X}}

\newcommand{\bB}{\mathbf{B}}

\newcommand{\bR}{\mathbf{R}}

\newcommand{\bfzero}{\mathbf{0}}

\newcommand{\bfmu}{\bm{\mu}}

\newcommand{\bfnu}{\bm{\nu}}

\newcommand{\bflambda}{\bm{\lambda}}

\newcommand{\bfSigma}{\bm{\Sigma}}

\newcommand{\tr}{tr}
\newcommand{\GP}{\mathcal{GP}}

\newcommand{\order}{\mathcal{O}}
\newcommand{\normal}{\mathcal{N}}

\DeclareMathOperator*{\argmin}{arg\,min}
\DeclareMathOperator*{\argmax}{arg\,max}
\DeclareMathOperator*{\KL}{KL}
\DeclareMathOperator*{\ELBO}{ELBO}
\DeclareMathOperator*{\bbE}{\mathbb{E}}

\newcommand{\sparsity}{\mathcal{S}}
\newcommand{\anc}{\mathcal{A}}

\DeclareSymbolFont{symbolsC}{U}{pxsyc}{m}{n}
\SetSymbolFont{symbolsC}{bold}{U}{pxsyc}{bx}{n}
\DeclareFontSubstitution{U}{pxsyc}{m}{n}
\DeclareMathSymbol{\medcirc}{\mathbin}{symbolsC}{7}

\usepackage{mathtools}
\mathtoolsset{showonlyrefs} 

\usepackage[textsize=tiny]{todonotes}

\icmltitlerunning{Variational SIC-LGP Approximation via Double KL Minimization}

\begin{document}

\twocolumn[
\icmltitle{Variational Sparse Inverse Cholesky Approximation for Latent Gaussian Processes via Double Kullback-Leibler Minimization}

\icmlsetsymbol{equal}{*}

\begin{icmlauthorlist}
\icmlauthor{Jian Cao}{equal,zzz}
\icmlauthor{Myeongjong Kang}{equal,yyy}
\icmlauthor{Felix Jimenez}{yyy}
\icmlauthor{Huiyan Sang}{yyy}
\icmlauthor{Florian Sch\"afer}{comp}
\icmlauthor{Matthias Katzfuss}{zzz}
\end{icmlauthorlist}

\icmlaffiliation{yyy}{Department of Statistics, Texas A\&M University, College Station, TX, USA}
\icmlaffiliation{zzz}{Department of Statistics and Institute of Data Science, Texas A\&M University, College Station, TX, USA}
\icmlaffiliation{comp}{School of Computational Science and Engineering, Georgia Institute of Technology, Atlanta, GA, USA}

\icmlcorrespondingauthor{Matthias Katzfuss}{katzfuss@gmail.com}

\icmlkeywords{maximum-minimum-distance ordering; nearest neighbors; variational inference; Vecchia approximation}

\vskip 0.3in
]

\printAffiliationsAndNotice{\icmlEqualContribution} 

\begin{abstract}
To achieve scalable and accurate inference for latent Gaussian processes, we propose a variational approximation based on a family of Gaussian distributions whose covariance matrices have sparse inverse Cholesky (SIC) factors.
We combine this variational approximation of the posterior with a similar and efficient SIC-restricted Kullback-Leibler-optimal approximation of the prior.
We then focus on a particular SIC ordering and nearest-neighbor-based sparsity pattern resulting in highly accurate prior and posterior approximations. For this setting, our variational approximation can be computed via stochastic gradient descent in polylogarithmic time per iteration.
We provide numerical comparisons showing that the proposed double-Kullback-Leibler-optimal Gaussian-process approximation (DKLGP) can sometimes be vastly more accurate for stationary kernels than alternative approaches such as inducing-point and mean-field approximations at similar computational complexity.
\end{abstract}

\section{Introduction}
\label{sec:intro}

Gaussian process (GP) priors are popular models for unknown functions in a variety of settings, including geostatistics \citep[e.g.,][]{Stein1999,Banerjee2004,Cressie2011}, computer model emulation \citep[e.g.,][]{Sacks1989,Kennedy2001,gramacy2020surrogates}, and machine learning \citep[e.g.,][]{Rasmussen2006,deisenroth2010efficient}. Latent GP (LGP) models, such as generalized GPs, assume a Gaussian or non-Gaussian distribution for the data conditional on a GP \citep[e.g.,][]{Diggle1998,chan2011generalized}. LGPs extend GPs to a large class of settings, including noisy, categorical, and count data. However, LGP inference is generally analytically intractable and hence requires approximations. In addition, direct GP inference is prohibitive for large datasets due to cubic scaling in the data size.
There are two main challenges for (L)GPs in many applications: One is to specify or learn a suitable kernel for the GP, and the other is carrying out fast inference for a given kernel. In this paper, we make no contributions to the former and instead focus on the latter challenge: We assume that a parametric kernel form is given and propose an efficient approximation method for LGP inference via structured variational learning. 

Many approaches to scaling GPs to large datasets were reviewed in \citet{Heaton2017} and \citet{Liu2018}, including low-rank approaches with a small number of pseudo points that are popular in machine learning.
Such low-rank GP approximations have been combined with variational inference for GPs \citep[e.g.,][]{Titsias2009,Hensman2013} and LGPs \citep[e.g.,][]{Hensman2015,leibfried2020tutorial}.

A highly promising approach to achieve GP scalability is given by nearest-neighbor Vecchia approximations from spatial statistics \citep[e.g.,][]{Vecchia1988,Stein2004,Datta2016,Katzfuss2017a}, which are optimal with respect to forward Kullback-Leibler (KL) divergence under the restriction of sparse inverse Cholesky (SIC) factors of the covariance matrix \citep{Schafer2020}.
Such SIC approximations have several attractive properties \citep[e.g., as reviewed by][]{Katzfuss2020}. They result in a valid joint density function given by the product of univariate conditional Gaussians, each of which can be independently computed in cubic complexity in the number of neighbors. This allows straightforward mini-batch subsampling with unbiased gradient estimators \citep{Cao2022}.
For the ordering and sparsity pattern used here, the number of neighbors needs to grow only polylogarithmically with the data size to achieve $\epsilon$-accurate approximations for Mat\'ern-type kernels up to boundary effects \citep{Schafer2020} due to the screening effect \citep{Stein2011}.
Many existing GP approximations, including low-rank and partially-independent conditional approaches, can be viewed as special cases of SIC approximations corresponding to particular orderings and sparsity patterns \citep{Katzfuss2017a}. 
SIC approximation using our ordering and sparsity pattern does not exhibit the same limitations as low-rank approximations \citep[][]{Stein2013a} and can hence be significantly more accurate for non-latent (i.e., directly observed) GPs \citep{Cao2022}.

SIC approximations of LGPs are more challenging. For LGPs with Gaussian noise, applying SIC approximations to the noisy responses reduces accuracy, and SIC approximations of the latent field may not be scalable \citep[e.g.,][]{Katzfuss2017a}. Existing approaches addressing this challenge \citep{Datta2016,Katzfuss2017a,Schafer2020,Geoga2022} do not consider estimation using stochastic gradient descent (SGD). 
For non-Gaussian LGPs, Laplace SIC approximations \citep{Zilber2019} are straightforward but can be inaccurate.
\citet{liu2019amortized} combined an SIC-type approximation to the prior with variational inference based on a variational family of Gaussians with a sparse Cholesky factor of the covariance matrix, but we are not aware of results guaranteeing that the covariance-Cholesky factor exhibits (approximate) sparsity under random ordering.
\citet{Wu2022} combined SIC-type approximations of LGPs with mean-field variational inference, but the latter may be inaccurate when there are strong correlations in the GP posterior \citep{mackay1992practical}.

To achieve scalable and accurate inference for LGPs, we propose a variational family of SIC Gaussian distributions and combine it with a SIC approximation to the GP prior (see Figure \ref{fig:diagram}).
Our approach is double-KL-optimal in the sense that variational approximation is reverse-KL-optimal for a given log normalizer (i.e., evidence) and our prior SIC approximation, which is available in closed form, is forward-KL-optimal for a given sparsity pattern \citep{Schafer2020}. 
\begin{figure}
    \begin{center}
\includegraphics[width=0.9\linewidth]{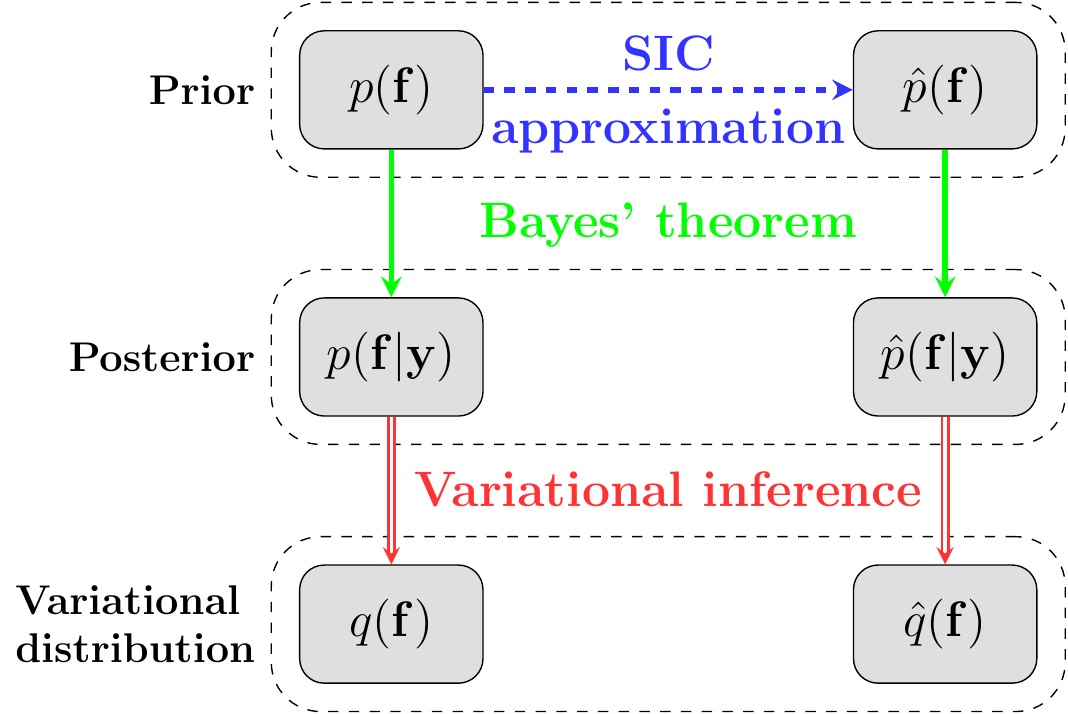}
    \caption{Double KL minimization for approximating the posterior distribution of a latent Gaussian field $\mathbf{f}$ given data $\by$: Based on a forward-KL-optimal SIC approximation $\hat p(\mathbf{f})$ of the prior, we obtain an SIC-restricted reverse-KL-optimal variational approximation $\hat q(\mathbf{f})$ to the posterior.}
    \label{fig:diagram}
    \end{center}
\end{figure}
Within our double-Kullback-Leibler-optimal Gaussian-process framework (DKLGP), we then focus on a particular ordering and nearest-neighbor-based sparsity pattern resulting in highly accurate prior and posterior approximations. We adopt a novel computational trick based on the concept of reduced ancestor sets for achieving efficient and scalable LGP inference. For this setting, our variational approximation can be computed via SGD in polylogarithmic time per iteration. While inducing-point methods assume that unobserved points depend on data only through inducing points \citep[e.g.,][]{frigola2014variational,Hensman2015}, our method allows fast and accurate KL-optimal prediction based on the screening effect. Our numerical comparisons show that DKLGP can be vastly more accurate than state-of-the-art alternatives such as inducing-point and mean-field approximations at a similar computational complexity.

\section{Methodology \label{sec:methodology}}

\subsection{Model}

Assume we have a vector $\by = (y_1,\ldots,y_n)^\top$ of noisy observations of a latent GP $f(\cdot) \sim \GP(\mu,K)$ at inputs $\bx_1,\ldots,\bx_n \in \mathbb{R}^{d}$, such that $p(\by|\mathbf{f}) = \prod_{i=1}^n p(y_i|f_i)$, where 
\begin{equation}
\label{eq:prior}
\mathbf{f} = (f_1,\ldots,f_n)^\top \sim \normal_n(\bfmu,\bK)
\end{equation}
with $\bfmu_i = \mu(\bx_i)$ and $\bK_{ij} = K(\bx_i,\bx_j)$. Throughout, we view the inputs $\bx_i$ as fixed (i.e., non-random) and hence do not explicitly condition on them.

Unless $\by|\mathbf{f}$ follows a Gaussian distribution, inference (such as computing the posterior $p(\mathbf{f}|\by)$) generally cannot be carried out in closed form. In addition, even for Gaussian likelihoods, direct inference scales as $\order(n^3)$ and is thus computationally infeasible for large $n$.
To address these challenges, we propose an approximation based on double KL minimization.

\subsection{Variational Sparse Inverse Cholesky Approximation\label{sec:varsic}}

Consider a lower-triangular sparsity pattern $\sparsity^q \subset \{1,\ldots,n\}^2$, with $\{(i,i): i=1,\ldots,n\} \subset \sparsity^q$ and such that $i \geq j$ for all $(i,j) \in \sparsity^q$. Our preferred choice of $\sparsity^q$ will be discussed in Section \ref{sec:maximin}, but typically we will have $(i,j) \in \sparsity^q$ if $\bx_i$ and $\bx_j$ are ``close.'' Corresponding to $\sparsity^q$, define the family of distributions $\mathcal{Q} = \{\normal_n(\bfnu,(\bV\bV^\top)^{-1}): \bfnu \in \mathbb{R}^n, \bV \in \mathbb{R}^{n\times n}, \bV \in \sparsity^q\}$, where we write $\bV \in \sparsity^q$ if $(i,j) \in \sparsity^q$ for all $\bV_{ij} \neq 0$. It is straightforward to show that any $q \in \mathcal{Q}$ can be represented in ordered conditional form as $q(\mathbf{f}) = \prod_{i=1}^n q(f_i|\mathbf{f}_{s_i^q})$, where $s_i^q = \{j > i: (j,i) \in \sparsity^q\}$ for $i=1,\ldots,n-1$ and $s_n^q = \emptyset$.

We approximate the posterior $p(\mathbf{f}|\by)$ by the closest distribution in $\mathcal{Q}$ in terms of reverse KL divergence:
\begin{equation}
    \label{eq:klposterior}
    \hat{q}(\mathbf{f}) = \argmin_{q \in \mathcal{Q}} \KL\big( q(\mathbf{f}) \big\| p(\mathbf{f}|\by) \big).
\end{equation}
We have $\KL( q(\mathbf{f}) \| p(\mathbf{f}|\by) ) = \log p(\by) - \ELBO(q)$, where $p(\by)$ does not depend on $q$, and so $\hat{q}$ satisfies
\begin{equation}
    \hat{q}(\mathbf{f}) 
    = \argmax_{q \in \mathcal{Q}} \ELBO(q) \label{eq:ELBOmin}.
\end{equation}

\begin{proposition}\label{prop:elbo}
The ELBO in \eqref{eq:ELBOmin} can be written up to an additive constant of $n/2$ as
\begin{align}
    \displaystyle
    \ELBO(q) 
    ={} &\sum_{i=1}^n \big(\bbE_{q} \log p(y_i|f_i) - ( (\bfnu-\bfmu)^\top \bL_{:,i} )^2/2 \\
    &+ \log (\bV_{ii}^{-1} \bL_{ii}) - \|\bV^{-1}\bL_{:,i}\|^2/2 \big),
\label{eq:ELBO}
\end{align}
where $\bL$ is the inverse Cholesky factor of $\bK$ such that $\bK^{-1} = \bL \bL^\top$, and $\bL_{:,i}$ denotes its $i$th column. 
\end{proposition}
All proofs can be found in Appendix~\ref{app:proofs}.

\subsection{Approximating the Prior via a Second KL Minimization}

Even for a sparse $\bV$, computing the ELBO in \eqref{eq:ELBO} is prohibitively expensive for large $n$, because computing $\bL$ (or any of its columns) from $\bK$ generally requires $\order(n^3)$ time. To avoid this, we replace the prior $p(\mathbf{f})$ defined in \eqref{eq:prior} by a Gaussian distribution that minimizes a second KL divergence under an SIC constraint.

Specifically, consider a second lower-triangular sparsity pattern $\sparsity^p \subset \{1,\ldots,n\}^2$, which may be the same as $\sparsity^q$. We define the corresponding set of distributions $\mathcal{P} = \{\normal_n(\tilde\bfmu,(\tilde\bL\tilde\bL^\top)^{-1}): \tilde\bfmu \in \mathbb{R}^n, \tilde\bL \in \mathbb{R}^{n\times n}, \tilde\bL \in \sparsity^p\}$.
We approximate the prior $p(\mathbf{f})$ by the closest approximation in $\mathcal{P}$ in terms of forward KL divergence:
\begin{equation}
    \label{eq:klprior}
    \hat{p}(\mathbf{f}) = \argmin_{\tilde{p} \in \mathcal{P}} \KL\big( p(\mathbf{f}) \big\| \tilde{p}(\mathbf{f}) \big).
\end{equation}
By a slight extension of \citet[][Thm.~2.1]{Schafer2020}, we can show that this optimization problem has an efficient closed-form solution. 

\begin{proposition}\label{prop:prior}
The solution to \eqref{eq:klprior} is $\hat{p}(\mathbf{f}) = \normal_n(\mathbf{f}|\bfmu,(\hat\bL\hat\bL^\top)^{-1})$, where the nonzero entries of the $i$th column of $\hat\bL$ can be computed in $\order(|\sparsity^p_i|^3)$ time as
\begin{equation}
\label{eq:priorklmin}
 \hat\bL_{\sparsity^p_i,i} = \bb_i (\bb_{i,1})^{-1/2}, \quad \text{ with } \bb_i = \bK^{-1}_{\sparsity^p_i,\sparsity^p_i} \be_1,
\end{equation}
and $\sparsity^p_i = \{j: (j,i) \in \sparsity^p\}$ is an ordered set with elements in increasing order (i.e., the first element is $i$).
\end{proposition}
Throughout, we denote by $\be_i$ a vector whose $i$th entry is one and all others are zero, and we index matrices before inverting so that $\bK^{-1}_{\sparsity^p_i,\sparsity^p_i} \colonequals (\bK_{\sparsity^p_i,\sparsity^p_i})^{-1}$.

The approximation in Proposition \ref{prop:prior} is equivalent to an ordered conditional approximation \citep{Vecchia1988} of the prior density $p(\mathbf{f}) = \prod_{i=1}^n p(f_i|\mathbf{f}_{(i+1):n})$ by:
\[
\textstyle
\hat{p}(\mathbf{f}) = \prod_{i=1}^n p(f_i|\mathbf{f}_{s_i^p}) = \prod_{i=1}^n \normal(f_i | \eta_i , \sigma^2_i),
\]
where
$
\eta_i = \bfmu_i - \hat\bL^{\top}_{s_i^p,i} ( \mathbf{f}_{s_i^p} - \bfmu_{s_i^p} ) / \hat\bL_{i,i}
$
and
$
\sigma^2_i = \hat\bL_{i,i}^{-2}
$,
with $s_i^p = \sparsity^p_i \setminus \{i\}$.

\subsection{Computing the ELBO based on Ancestor Sets\label{sec:ancestors}}

Plugging $\hat{p}(\mathbf{f})$ into \eqref{eq:ELBOmin}, the ELBO in \eqref{eq:ELBO} becomes
\begin{align}
    \displaystyle
    \ELBO(q) 
    ={} &\sum_{i=1}^n \big(\bbE_{q} \log p(y_i|f_i) - ( (\bfnu-\bfmu)^\top \hat\bL_{:,i} )^2/2 \\
    &+ \log (\bV_{ii}^{-1} \hat\bL_{ii}) - \|\bV^{-1}\hat\bL_{:,i}\|^2/2 \big),
    \label{eq:ELBOhat}
\end{align}
with the $i$th summand depending on $\hat\bL$ only via its $i$th column $\hat\bL_{:,i}$, whose nonzero entries can be computed in $\order(|\sparsity^p_i|^3)$ time using \eqref{eq:priorklmin}. 

We need to compute $\bV^{-1}\hat\bL_{:,i}$ and $\bV^{-1}\be_i$, the latter of which appears in $\bbE_{q} \log p(y_i|f_i)$ (see Section \ref{sec:optim}). The nonzero entry of $\be_i$ (i.e., $\{i\}$) is a subset of the nonzero entries of $\hat\bL_{:,i}$ (i.e., $\sparsity^p_i$), and hence we focus our discussion on computing $\bV^{-1}\hat\bL_{:,i}$. Solving this sparse triangular system in principle requires $\order(|\sparsity^q|)$ time.

However, it is possible to speed up computation by omitting rows and columns of $\bV$ that do not correspond to the ancestor set $\anc_i$ of $\sparsity^p_i$ with respect to $\sparsity^q$, which is defined as $\anc_i = \big\{j \ge i: \text{ there exists a path } \mathcal{L} = \{(j, l_1),(l_1,l_2),\ldots,(l_{a-1},l_a),(l_a,l) \} \subset \sparsity^q \text{ for some } l \in \sparsity^p_i\big\}$. Ancestor sets are properties of the directed acyclic graphs that can be used to represent our triangular sparsity structures, as illustrated in Appendix \ref{app:graph}.

\begin{proposition}\label{prop:ancestor}
$(\bV^{-1}\hat\bL_{:,i})_j = 0$ for all $j \notin \anc_i$.
\end{proposition}

Thus, we have
\begin{equation}
\|\bV^{-1}\hat\bL_{:,i}\| =
\|\bV_{\anc_i,\anc_i}^{-1}\hat\bL_{\anc_i,i}\|,
\label{eq:ancestorelboterm}
\end{equation}
where $\bV_{\anc_i,\anc_i}^{-1}\hat\bL_{\anc_i,i}$ can be computed in $\order(|\anc_i| |\sparsity^q_i|)$ time.

\subsection{Maximin Ordering and Nearest-neighbor Sparsity\label{sec:maximin}}

\begin{figure*}[htbp]
\centering
	\begin{subfigure}{.33\textwidth}
	\centering
 	\includegraphics[width =.99\linewidth]{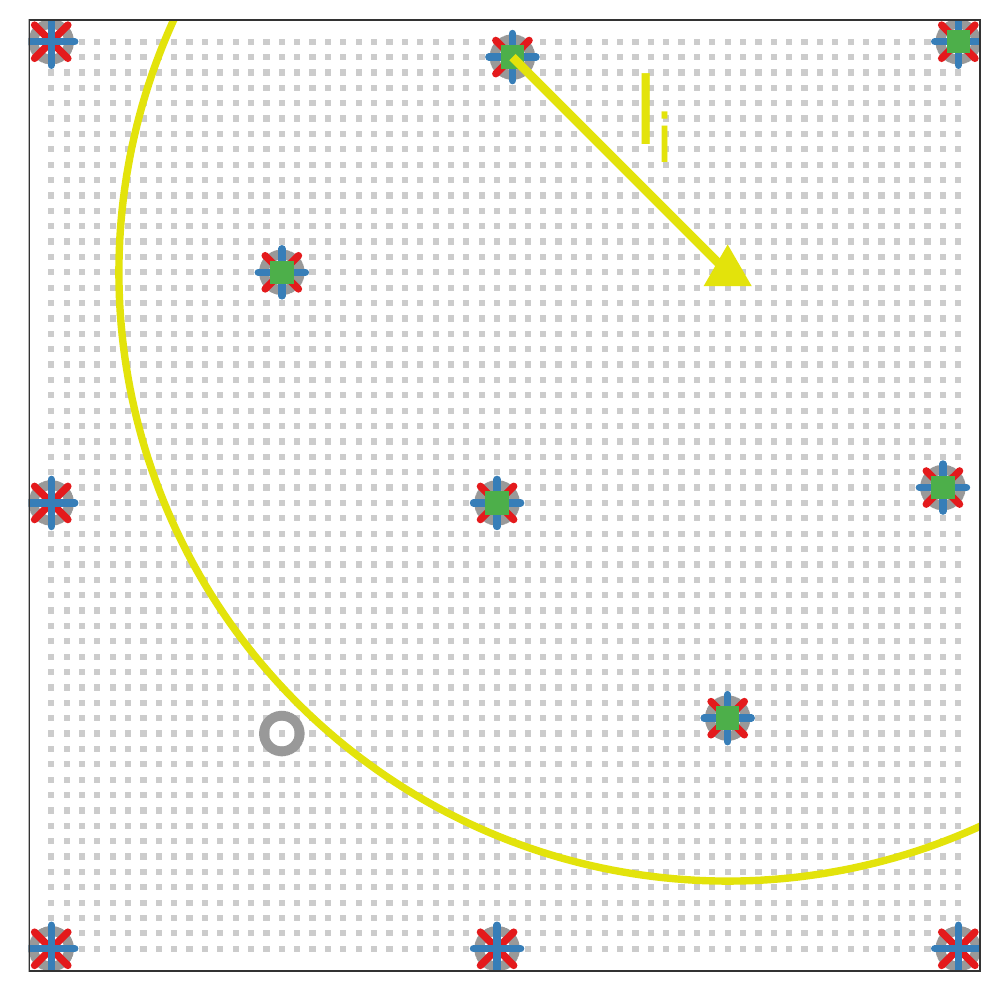}
	\caption{$i=n-12$}
	\label{fig:mm1}
	\end{subfigure}%
\hfill
	\begin{subfigure}{.33\textwidth}
	\centering
 	\includegraphics[width =.99\linewidth]{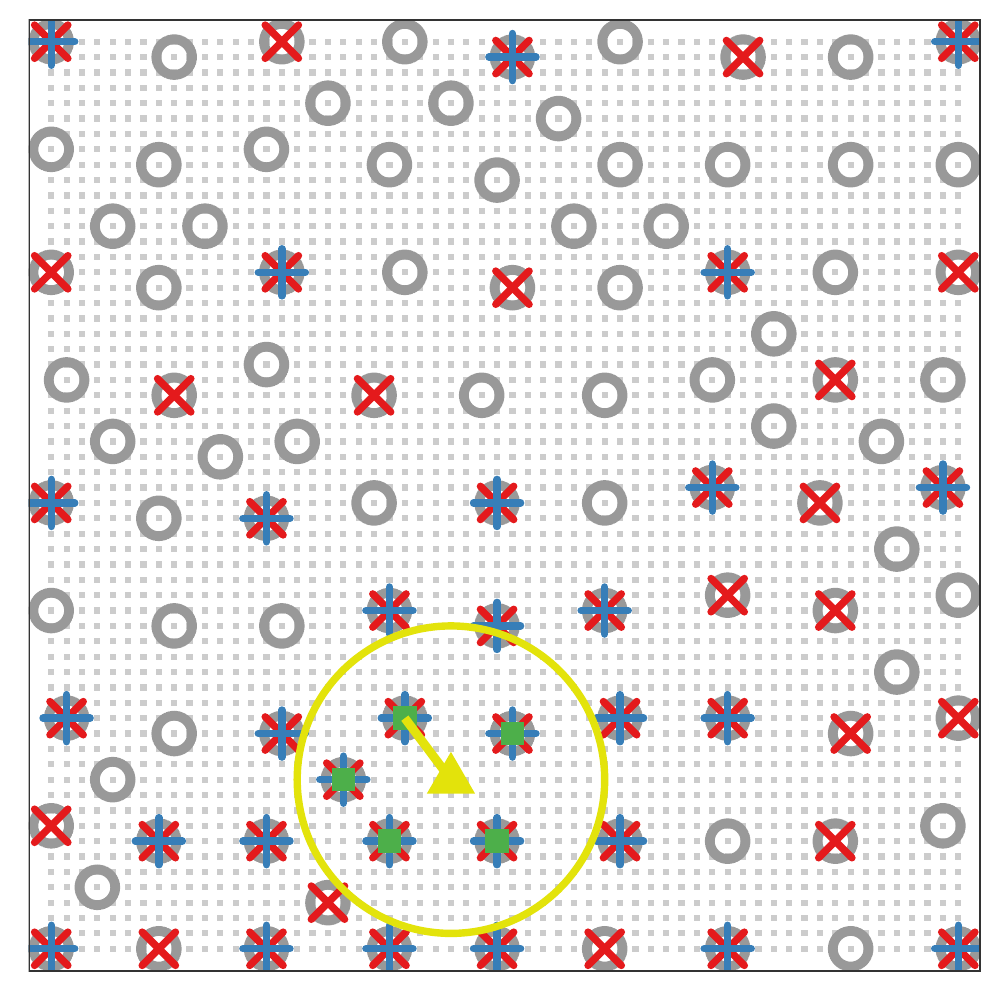}
	\caption{$i=n-100$}
	\label{fig:mm2}
	\end{subfigure}%
\hfill
	\begin{subfigure}{.33\textwidth}
	\centering
 	\includegraphics[width =.99\linewidth]{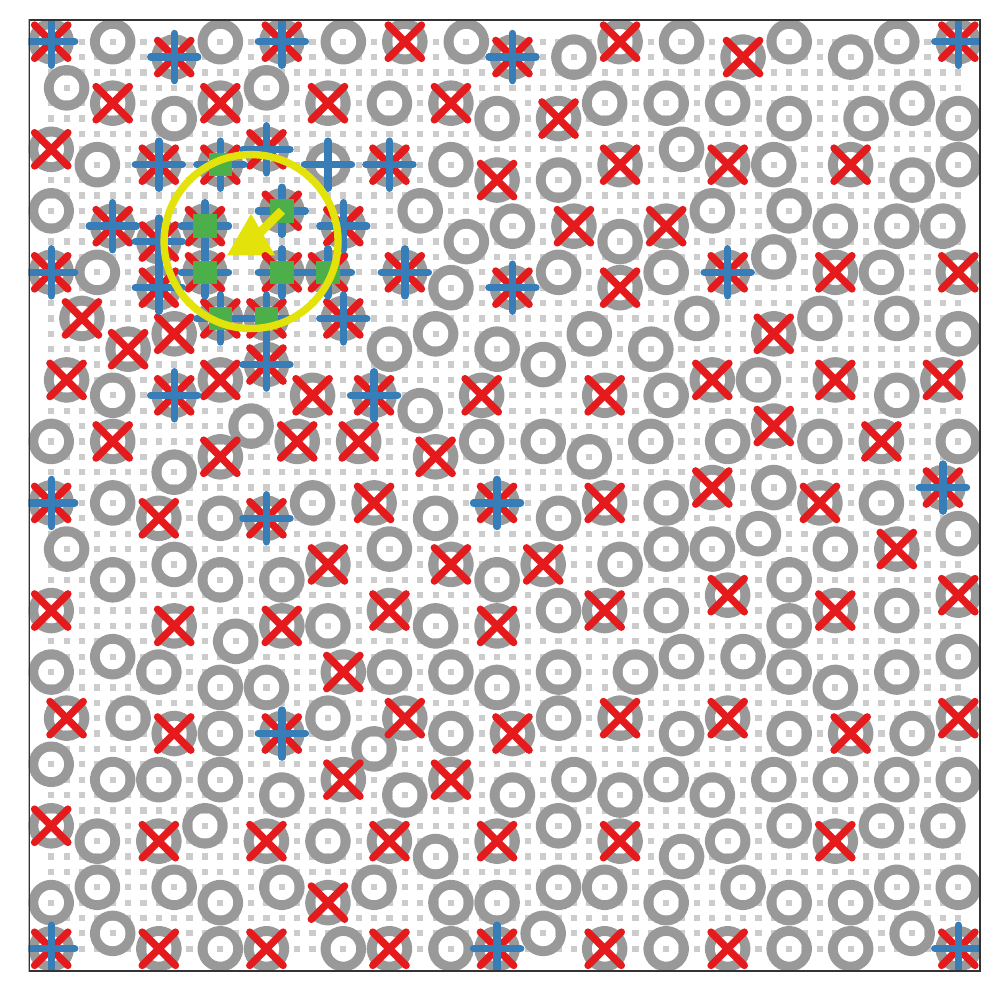}
	\caption{$i=n-289$}
	\label{fig:mm3}
	\end{subfigure}%
	
  \caption{Reverse maximin ordering on a grid (small gray dots) of size $n=60 \times 60 = 3{,}600$ on a square. For three different indices $i$, we show the $i$th ordered input (${\color[HTML]{e3e30b}\blacktriangle}$), the subsequently ordered $n-i$ inputs (${\color{black}\bm{\medcirc}}$), the distance $\ell_i$ to the nearest neighbor (${\color[HTML]{e3e30b}\boldsymbol{-}}$), the neighboring subsequent inputs $\sparsity_i$ (${\color[HTML]{4DAF4A}\blacksquare}$) within a (yellow) circle of radius $\rho \ell_i$ (here, $\rho=2$), the reduced ancestors $\tilde{\anc}_i$ (${\color[HTML]{377EB8}\bm{+}}$), and the ancestors $\anc_i$ (${\color[HTML]{E41A1C}\bm{\times}}$).
  }
\label{fig:maxmin}
\end{figure*}

\citet{Schafer2020} proposed a sparsity pattern $\sparsity$ based on reverse-maximum-minimum-distance (r-maximin) ordering (see Figure \ref{fig:maxmin} for an illustration). R-maximin ordering picks the last index $i_n$ arbitrarily (often in the center of the input domain), and then the previous indices are sequentially selected for $k = n-1, n-2, \ldots, 1$ as
$
i_k = \argmax_{i \, \notin \, \mathcal{I}_{k}} \,\, \min_{j \, \in \, \mathcal{I}_{k}} \text{dist}(\bx_i,\bx_j)
$, 
where $\mathcal{I}_{k} = \{i_{k+1} , \ldots , i_{n}\}$. 
Throughout, we assume that our indexing follows r-maximin ordering (e.g., $f_k = f_{i_k}$).
We can then define the sparsity pattern by $\sparsity_i =\{ j \geq i: \text{dist}(\bx_i,\bx_j)\leq \rho \ell_i\}$, for some fixed $\rho \geq 1$, where $\ell_{i} = \min_{j > i} \text{dist}(\bx_i,\bx_j)$. We can compute $\text{dist}(\bx_i,\bx_j)$ as Euclidean distance between the inputs, potentially in a transformed input space (see Section \ref{sec:optim} for more details).
The conditioning sets are all of approximately size $|\sparsity_i|=\mathcal{O}(\rho^d) \approx m = |\sparsity|/n$ under mild assumptions on the regularity of the inputs. \citet{Schafer2020} proved that an $\epsilon$-accurate approximation of the prior can be obtained using $\sparsity^p=\sparsity$ with $\rho = \order(\log (n / \epsilon))$ for kernels $K$ that are Green's functions of elliptic boundary-value problems (similar to Mat\'ern kernels up to boundary effects) and demonstrated high numerical accuracy of the posterior using $\sparsity^q=\sparsity$ for Gaussian likelihoods. For non-Gaussian likelihoods, this implies highly accurate approximations to the posterior when a second-order Taylor expansion can adequately approximate the posterior.

While this means that our DKLGP can achieve high accuracy by choosing $\sparsity^p = \sparsity^q = \sparsity$, the resulting ancestor sets can grow roughly linearly with $n$ (e.g., see Figure \ref{fig:sizes_n_rho}). Hence, even evaluating the ELBO based on the ancestor sets would often be prohibitively expensive for large $n$. 
However, it is possible to ignore most ancestors in \eqref{eq:ancestorelboterm} and only incur a small approximation error. 
Specifically, consider reduced ancestor sets $\tilde{\anc}_i = \{ j \geq i: \text{dist}(\bx_i,\bx_j)\leq \rho \ell_j\}$, where the last subscript is now a $j$, not an $i$. As illustrated in Figure~\ref{fig:maxmin}, we have $\sparsity_i \subset \tilde{\anc}_i$ (because $\ell_j \geq \ell_i$ for $j \geq i$) and approximately $\tilde{\anc}_i \subset \anc_i$. The reduced ancestor sets are of size $|\tilde{\anc}_i| = \order(\rho^d \log n) = \order(m \log n)$ and can all be computed together in $\order(nm\log^2 n)$ time \citep{Schafer2017}. Hence, reduced ancestor sets can be orders of magnitude smaller than full ancestor sets (see Figures \ref{fig:sizes_n_rho} and \ref{fig:sizes_rho}).

\begin{figure*}
\centering
\hfill
\begin{subfigure}{.33\textwidth}
\centering
\includegraphics[width =.99\linewidth]{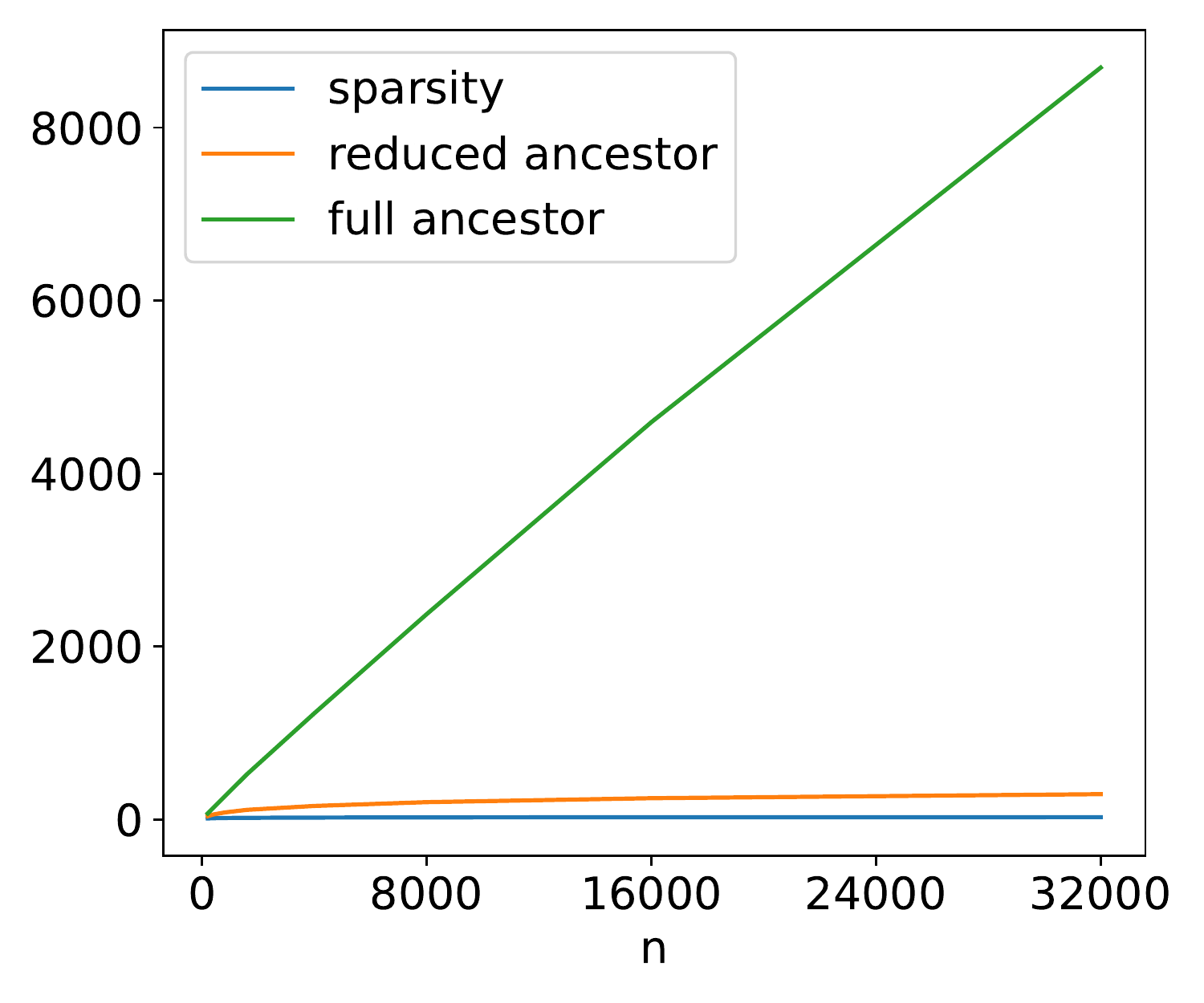}
\caption{Average sizes of $\sparsity_i$, $\tilde\anc_i$, and $\anc_i$}
\label{fig:sizes_n_rho}
\end{subfigure}
\hfill
\begin{subfigure}{.33\textwidth}
\centering
\includegraphics[width =.99\linewidth]{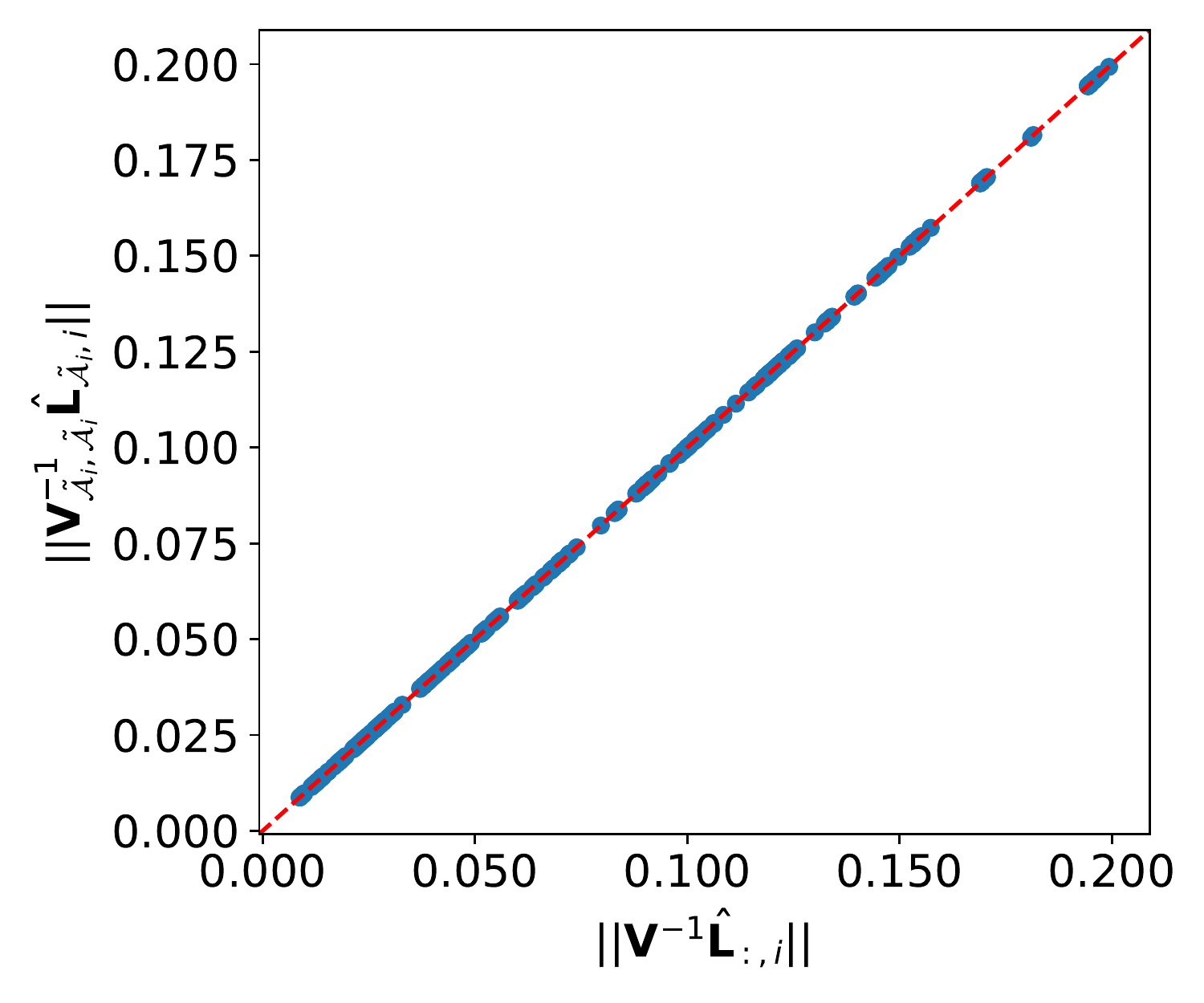}
\caption{$\|\bV_{\tilde\anc_i,\tilde\anc_i}^{-1}\hat\bL_{\tilde\anc_i,i}\|$ vs $\|\bV^{-1}\hat\bL_{:,i}\|$}
\label{fig:reducedsolve}
\end{subfigure}
\hfill
\begin{subfigure}{.33\textwidth}
\centering
\includegraphics[width =.99\linewidth]{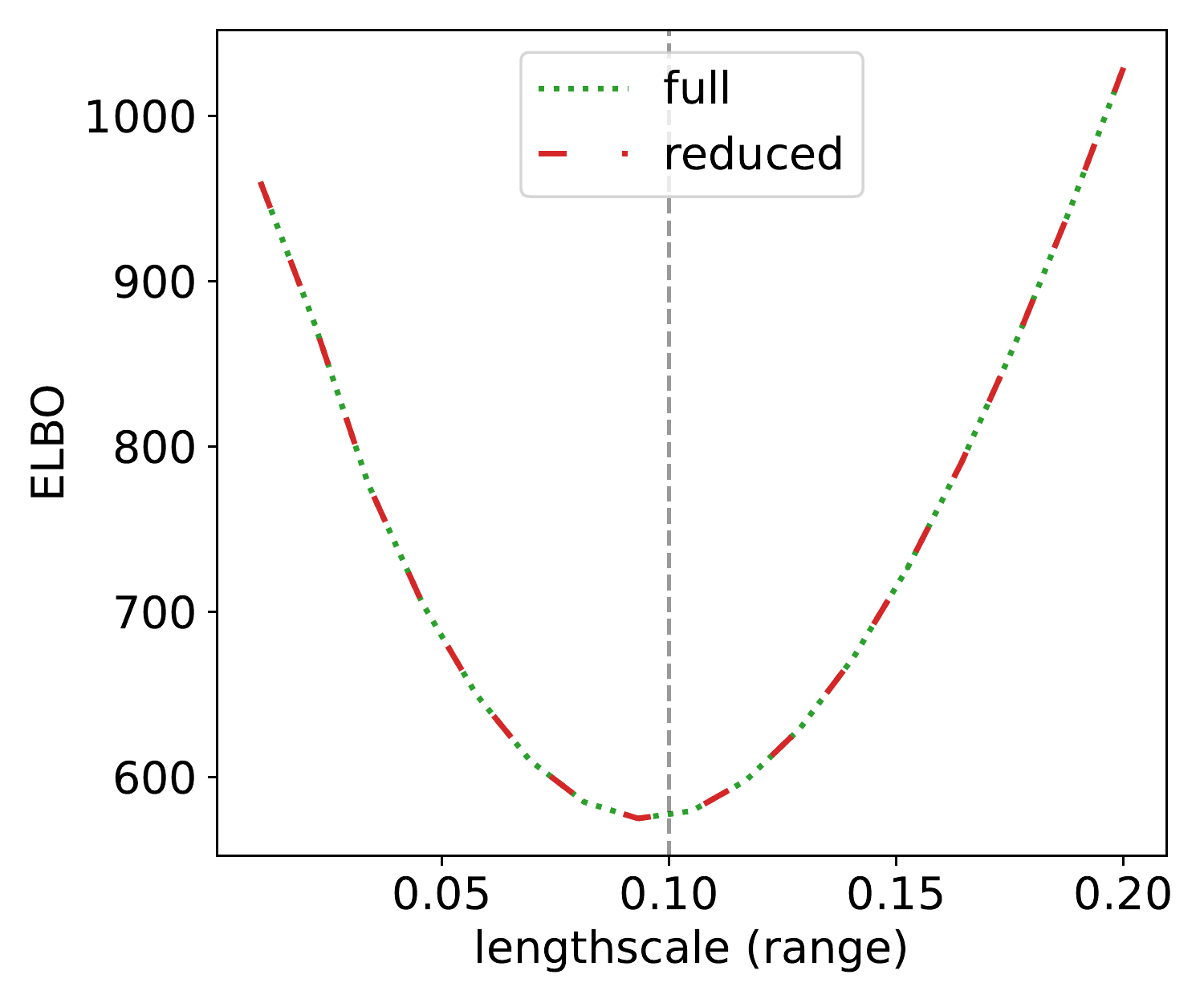}
\caption{ELBO using reduced ancestors}
\label{fig:ELBO_err}
\end{subfigure}
\hfill
\caption{Reduced ancestor sets are much smaller than full ancestor sets, as shown in (a), and hence greatly reduce computational cost, but result in negligible approximation error in the ELBO, as shown in (b) and (c). 
Specifically, (a) shows average sizes of the sparsity sets $\sparsity_i$, reduced ancestor sets $\tilde\anc_i$, and full ancestor sets $\anc_i$ as a function of $n$ with $d=5$; for $n=32{,}000$, we have $|\sparsity_i| = 30$, $|\tilde\anc_i| = 293$, and $|\anc_i| = 8{,}693$. 
(b) compares $\|\bV_{\tilde\anc_i,\tilde\anc_i}^{-1}\hat\bL_{\tilde\anc_i,i}\|$ with reduced ancestor sets versus $\|\bV^{-1}\hat\bL_{:,i}\|$  for $i = 1, \ldots, n$, where $n = 500$ and $d = 2$.
(c) compares ELBO curves based on full \eqref{eq:ELBOhat} and reduced \eqref{eq:ELBOreduced} ancestor sets, as functions of the range parameter with true value $0.1$, for $n=500$ and $d=2$. In all plots, we set $\rho=2$ and the $n$ inputs are sampled uniformly on $[0,1]^d$.}
\label{fig:effectofreduced}
\end{figure*}

\begin{claim}\label{claim:reducedancestor}
For Mat\'ern-type LGPs with exponential-family likelihoods,  $(\bV^{-1}\hat\bL_{:,i})_j \approx 0$ for all $j \notin \tilde\anc_i$, where $\bV$ minimizes the ELBO in \eqref{eq:ELBOhat}, under mild conditions.
\end{claim}
We provide a non-rigorous justification for this claim in Appendix~\ref{app:proofs}.
Together, Proposition \ref{prop:ancestor} and Claim \ref{claim:reducedancestor} imply that $\|\bV^{-1}\hat\bL_{:,i}\| \approx \|\bV_{\tilde\anc_i,\tilde\anc_i}^{-1}\hat\bL_{\tilde\anc_i,i}\|$ (as illustrated in Figure~\ref{fig:reducedsolve}), and so replacing the former by the latter in the ELBO causes negligible error (Figure~\ref{fig:ELBO_err}). Similar numerical results were obtained for two other popular kernels in Figures~\ref{fig:effectofreduced_RBF} and \ref{fig:effectofreduced_RQ} in Appendix~\ref{app:additional}, suggesting that our approach is applicable to beyond the Mat\'ern family.

\subsection{Optimization of the ELBO\label{sec:optim}}

The class of distributions $\mathcal{Q} = \{\normal_n(\bfnu,(\bV\bV^\top)^{-1}): \bfnu \in \mathbb{R}^n, \bV \in \mathbb{R}^{n\times n}, \bV \in \sparsity^q\}$ has $n$ parameters in $\bfnu$ and $|\sparsity|$ parameters in $\bV$. We propose to find the optimal $\hat{q} \in \mathcal{Q}$ by minimizing our approximation of $-\ELBO(q)$ with respect to these $\order(nm)$ unknown parameters via minibatch stochastic gradient descent. For each minibatch $\mathcal{B}$, this requires computing the gradient of
\begin{equation}\label{eq:ELBOreduced}
\begin{split}
    \sum_{i \in \mathcal{B}} \big( &\bbE_{q} \log p(y_i|f_i) - ( (\bfnu-\bfmu)^\top \hat\bL_{:,i} )^2/2 \\
    &+ \log (\bV_{ii}^{-1} \hat\bL_{ii}) - \|\bV_{\tilde\anc_i,\tilde\anc_i}^{-1}\hat\bL_{\tilde\anc_i,i}\|^2/2 \big) 
    \end{split} 
\end{equation}
using automatic differentiation.

For Gaussian observations with $y_i|f_i \sim \normal(f_i,\tau^2_i)$, we have
$ -2\bbE_{q} \log p(y_i|f_i) 
= \big( (y_i - \bfnu_i)^2 + \|\bV^{-1}\be_i\|^2 \big) /\tau_i^2 + \log \tau_i^2 + \log 2\pi$.
For more general distributions $p(y_i|f_i)$, we can use the Monte Carlo gradient estimator \citep{Kingma2014} and approximate
$ \bbE_{q} \log p(y_i|f_i) \approx (1/L) \sum_{l=1}^L p(y_i|f_i^{(l)})$,
where $f_i^{(l)} = \bfnu_i + (\bV^{-1}\be_i)^\top\bz^{(l)}$, $\bz^{(l)}  \stackrel{iid}{\sim} \normal_n(\bfzero,\bI_n)$, and $\bI_n$ is the $n \times n$ identity matrix.

Evaluating each summand in \eqref{eq:ELBOreduced} requires $\order(|\sparsity_i|^3) = \order(m^3)$ time for obtaining $\hat\bL_{:,i}$ and $\order(m^2 \log n)$ time for solving $\bV_{\tilde\anc_i,\tilde\anc_i}^{-1}\hat\bL_{\tilde\anc_i,i}$, because $|\tilde{\anc}_i| = \order(m \log n)$. The $\order(m^3)$ cost dominates, as we typically need $m = \order(\log^d n)$ for accurate approximations \citep{Schafer2020}; for example, in Figure~\ref{fig:sizes_n_rho}, $|\tilde\anc_i| |\sparsity_i|$ is smaller than $|\sparsity_i|^3$. Also, $\hat\bL$ does not need to be pre-computed and stored, as each column $\hat\bL_{:,i}$ can be computed ``on-the-fly''; this is especially useful for hyperparameter estimation, for which $p(\mathbf{f})$ and hence $\hat\bL$ changes with the hyperparameters at each gradient-descent iteration.

We initialize the optimization using an estimate of $\bfnu$ and $\bV$ based on a Vecchia-Laplace approximation \citep{Zilber2019} of $p(\mathbf{f} | \by)$ combined with an efficient incomplete Cholesky (IC0) approximation \citep{Schafer2020} of the posterior SIC factor.
While this initialization itself provides a reasonable approximation to the posterior, hyperparameter estimation for this approach is more difficult, and it is less accurate than DKLGP even for known hyperparameters as shown in Figure \ref{fig:post_mean_opt} in Appendix \ref{app:additional}.

The ordering and sparsity pattern in Section \ref{sec:maximin} depend on a distance metric, $\text{dist}(\bx_i,\bx_j)$, between inputs. We have found that the accuracy of the resulting approximation can be improved substantially by computing the Euclidean distance between inputs in a transformed input space in which the GP kernel is isotropic, as suggested by \citet{Katzfuss2020,Kang2021}. For example, consider an automatic relevance determination (ARD) kernel of the form $K(\bx_i,\bx_j) = K_{o} (q(\bx_i,\bx_j))$, where $K_{o}$ is an isotropic kernel (e.g., a Mat\'ern kernel with smoothness $1.5$ is used throughout this paper) and $q(\bx_i,\bx_j) = \| \bx^{\bflambda}_i - \bx^{\bflambda}_j \|$ is a Euclidean distance based on scaled inputs $\bx^{\bflambda} = (x_1/\lambda_1,\ldots,x_d/\lambda_d)$ with individual ranges or length-scales $\bflambda = (\lambda_1,\ldots,\lambda_d)$ for the $d$ input dimensions. In this example, we take $\text{dist}(\bx_i,\bx_j) = q(\bx_i,\bx_j)$ when computing the sparsity pattern.
When the scaled distance and hence the sparsity pattern depend on unknown hyperparameters (e.g., $\bflambda$ in the ARD case), we carry out a two-step optimization procedure: First, we run our ELBO optimization for a few epochs based on the sparsity pattern obtained using an initial guess of $\bflambda$ to obtain a rough estimate of $\bflambda$, which we then use to obtain the final ordering and sparsity pattern and warm-start our ELBO optimization.

\subsection{Prediction}
\label{subsec:prediction}
An important task for (L)GP models is prediction at unobserved inputs, meaning that we want to obtain the posterior distribution of latent GP variables $\mathbf{f}^*$ at new inputs $\bx_1^*,\ldots,\bx_{n^*}^*$ given the data $\by$. To do so, we consider the joint posterior distribution of $\tilde{\mathbf{f}} = (\mathbf{f}^*,\mathbf{f})$, from which any desired marginal distribution can be computed. Since working with the joint covariance matrix $\tilde{\bK}$ is again computationally prohibitive, we make a joint SIC assumption on the posterior distribution of $\tilde{\mathbf{f}}$ (with the prediction variables ordered first) that naturally extends the SIC assumption for $\mathbf{f}$ in $q(\mathbf{f})$. For the exact posterior, we have $$p(\tilde{\mathbf{f}}|\by) = p(\mathbf{f}^*|\mathbf{f},\by) p(\mathbf{f}|\by) = p(\mathbf{f}^*|\mathbf{f}) p(\mathbf{f}|\by).$$ Similarly, we assume $q(\tilde{\mathbf{f}}) = q(\mathbf{f}^*|\mathbf{f}) q(\mathbf{f})$, where $q(\mathbf{f}) = \normal_n(\mathbf{f}|\bfnu,(\bV\bV^\top)^{-1})$ was obtained as described in previous sections, and $q(\mathbf{f}^*|\mathbf{f})$ is a sparse approximation of $p(\mathbf{f}^*|\mathbf{f})$. For $i=1,\ldots,n^*$, let $\sparsity^*_i \subset \{i,i+1,\ldots,n^*+n\}$ denote the $i$th sparsity set relative to the joint posterior.

We define the approximation to the joint posterior by the minimizer of the expected forward-KL divergence between $p(\mathbf{f}^*|\mathbf{f})$ and $q(\mathbf{f}^*|\mathbf{f})$ for given $\bfnu$ and $\bV$, that is,
\begin{equation}
    \label{eq:klprediction}
    \hat{q}(\tilde{\mathbf{f}}) = \argmin_{q(\tilde{\mathbf{f}}) \in \tilde{\mathcal{Q}}(\bfnu, \bV)} \bbE_{p} \Big[ \KL \big( p(\mathbf{f}^* | \mathbf{f}) \big\| q(\mathbf{f}^* | \mathbf{f}) \big) \Big] ,
\end{equation}
where 
\begin{align}
    \tilde{\mathcal{Q}}(\bfnu , \bV) = \{ &\normal_{n^*+n} ((\bfnu^*{}^\top,\bfnu^\top)^\top , (\bV^*,(\bfzero, \bV^\top)^\top)): \\ &\ \bfnu^* \in \mathbb{R}^{n^*} , \bV^* \in \mathbb{R}^{(n^*+n) \times n^*}, \bV^* \in \sparsity^* \}
\end{align}
and $\sparsity^* = \bigcup_{i=1}^{n^*} \{(j,i): j \in \sparsity^*_i \}$. The resulting approximation can be obtained efficiently:
\begin{proposition}\label{prop:prediction}
For given $\bfnu$, $\bV$, and $\sparsity^*$,
$
\hat{q}(\tilde{\mathbf{f}}) =
\normal_{n^*+n}(\tilde{\mathbf{f}}|\tilde\bfnu,(\tilde\bV\tilde\bV^\top)^{-1}),
$
where $\tilde\bfnu = (\hat{\bfnu}^*{}^\top,\bfnu^\top)^\top$, $\tilde\bV = (\hat{\bV}^*,(\bfzero, \bV^\top)^\top)$, $\hat{\bV}^*=(\hat{\bV}^{**}{}^\top,\hat{\bV}^{o*}{}^\top)^\top$,
\begin{equation}
\label{eq:predictklmin}
    \hat{\bV}^*_{\sparsity^*_i,i} = \bc_i (\bc_{i,1})^{-1/2}, \quad \text{ with } \bc_i = K(\sparsity^*_i,\sparsity^*_i)^{-1} \be_1,
\end{equation}
\begin{equation}
    \hat{\bfnu}^* = \bfmu^* - (\hat{\bV}^{**})^{-\top} \hat{\bV}^{o*}{}^{\top}(\bfnu-\bfmu),
\end{equation}
and $\bfmu^* = (\mu(\bx_1^*),\ldots,\mu(\bx_{n^*}^*))^\top$.
\end{proposition}

The posterior distribution of a desired summary, say $\ba^\top\tilde{\mathbf{f}}$ can then be computed as $q(\ba^\top\tilde{\mathbf{f}}) = \normal(\ba^\top\tilde\bfnu,\|\tilde\bV^{-1}\ba\|^2)$. In particular, the marginal posterior of $f^*_i$ can be obtained using $\ba=\be_i$ as $q(\be_i^\top\tilde{\mathbf{f}}) = \normal(\bfnu^*_i,\|\tilde\bV^{-1}\be_i\|^2)$.

We again consider an r-maximin ordering and nearest-neighbor sparsity pattern similar to above, but now conditioned on the prediction points being ordered first, and the training points ordered after (in the same ordering as before).
Once the prediction points are in this conditional r-maximin ordering, we can define
$$\ell_{i}^* = 
\min_{i < j \le n^*} \text{dist}(\bx^*_{i},\bx^*_j) \wedge \min_{1 \le j \le n} \text{dist}(\bx^*_{i},\bx_j)$$ 
and 
\begin{align}
    \sparsity_i^* = \{ j \geq i : \text{dist}(\bx_i^* &, \bx_j^*)\leq \rho \ell_i^*\} \\ &\cup \{ j + n^* : \text{dist}(\bx_i^* , \bx_j)\leq \rho \ell_i^*\}.
\end{align}
This ordering and sparsity pattern can be computed rapidly and was shown to lead to highly accurate approximations; more details can be found in \citet[Section 4.2.1]{Schafer2020}. 
Note that while computing the prediction variances can be expensive, we can again approximate $\|\tilde\bV^{-1}\be_i\| \approx \|\tilde\bV_{\tilde\anc_i^*,\tilde\anc_i^*}^{-1}\be_{i;\tilde\anc_i^*}\|$ using a reduced ancestor set 
\begin{align}
    \tilde\anc_i^* = \{ j \geq i: \text{dist}(\bx_i^* &, \bx_j^*)\leq \rho \ell^*_j\} \\ &\cup \{ j + n^* : \text{dist}(\bx_i^* , \bx_j)\leq \rho \ell^*_j\},
\end{align}
where the last subscript is a $j$, not an $i$.

\section{Numerical Comparisons\label{sec:numerical}}

\subsection{Experimental Setup\label{sec:setup}}

We compared the following approaches:
\begin{description}[itemsep=1pt,topsep=2pt,parsep=1pt]
\item[DKLGP:] Our method with r-maximin ordering and nearest-neighbor sparsity pattern
\item[DKL-G:] Same as DKLGP but with global sparsity pattern $\sparsity_i^p = \sparsity_i^q = \{1, \ldots, m\}$
\item[DKL-D:] Same as DKLGP but with diagonal sparsity pattern $\sparsity^q_i = \{i\}$
\item[SVIGP:] Stochastic variational GP proposed by \citet{Hensman2013}
\item[VNNGP:] Variational nearest neighbor GP proposed by \citet{Wu2022}
\end{description}
In figures and tables, we use abbreviated acronyms DKL, SVI, and VNN to save space. SVIGP and VNNGP are two state-of-the-art variational GP methods, while DKL-G and DKL-D are variants of our DKLGP that resemble SVIGP and VNNGP, respectively. SVIGP assumes independence in $\mathbf{f}$ conditional on $m$ global inducing variables. VNNGP scales up the inducing points to be equal to the observed input locations, ensuring computational feasibility by assuming that each conditions only on $m$ others a priori, combined with a mean-field approximation to the posterior. We used the GPyTorch \citep{gardner2018gpytorch} implementations of SVIGP and VNNGP. For DKL-G and DKL-D, one can easily see that $\anc_i = \sparsity_i^p$, and so reduced ancestor sets are not necessary. For all methods, computing a term in the ELBO requires $\order(m^3)$ time per sample. (Reusing Cholesky factors for all samples in a minibatch is straightforward for SVIGP; similar savings may also be possible for the other methods based on the supernode ideas suggested by \citealp{Schafer2020}.) Hence, $m$ can be viewed as a comparable complexity parameter that trades off computational speed (for small $m$) against accuracy (large $m$). Thus, for our numerical comparison, we aligned the $m$ for all methods with the average size of $\sparsity_i$ for a given $\rho$.

Throughout, we assumed $f(\cdot) \sim \GP(0,K)$, where $K$ is a Mat\'ern1.5 ARD kernel whose variance (set to one for simulations) and range (i.e., length-scale) parameters $\bflambda$ were estimated. We considered three different likelihoods $p(y_i | f_i)$:
\begin{description}[itemsep=1pt,topsep=2pt,parsep=1pt]
    \item[Gaussian:] $y_i | f_i \sim \normal(f_i,\sigma_{\epsilon}^2)$
    \item[Student-$t$:] $y_i | f_i \sim \mathcal{T}_{2}(f_i,\sigma_{\epsilon}^2)$ with 2 degrees of freedom
    \item[Bernoulli-logit:] $y_i | f_i \sim \mathcal{B}((1 + e^{-f_i})^{-1})$
\end{description}
The noise variance $\sigma_{\epsilon}^2$ was estimated from the data; for simulations, we used $\sigma_{\epsilon}^2 = 0.1^2$ except where specified otherwise.

For estimation of hyperparameters, the initial values for $\bflambda$, $\sigma_{\epsilon}^2$, and the variance in $K$ were all $0.25$. DKLGP and its variants ran the Adam optimizer for $35$ epochs. SVIGP and VNNGP used natural gradient descent and Adam, respectively, as their optimizer for $500$ epochs as suggested in \citet{Wu2022}. The minibatch size was $128$ and a multi-step scheduler with a scaling factor of $0.1$ was used for all methods.

\subsection{Visual Comparison in One Dimension}
\label{subsec:1D}

\begin{figure*}
\centering
    \hfill
	\begin{subfigure}{.48\textwidth}
	\centering
 	\includegraphics[trim={0 0cm 0 0}, clip, width =.99\linewidth]{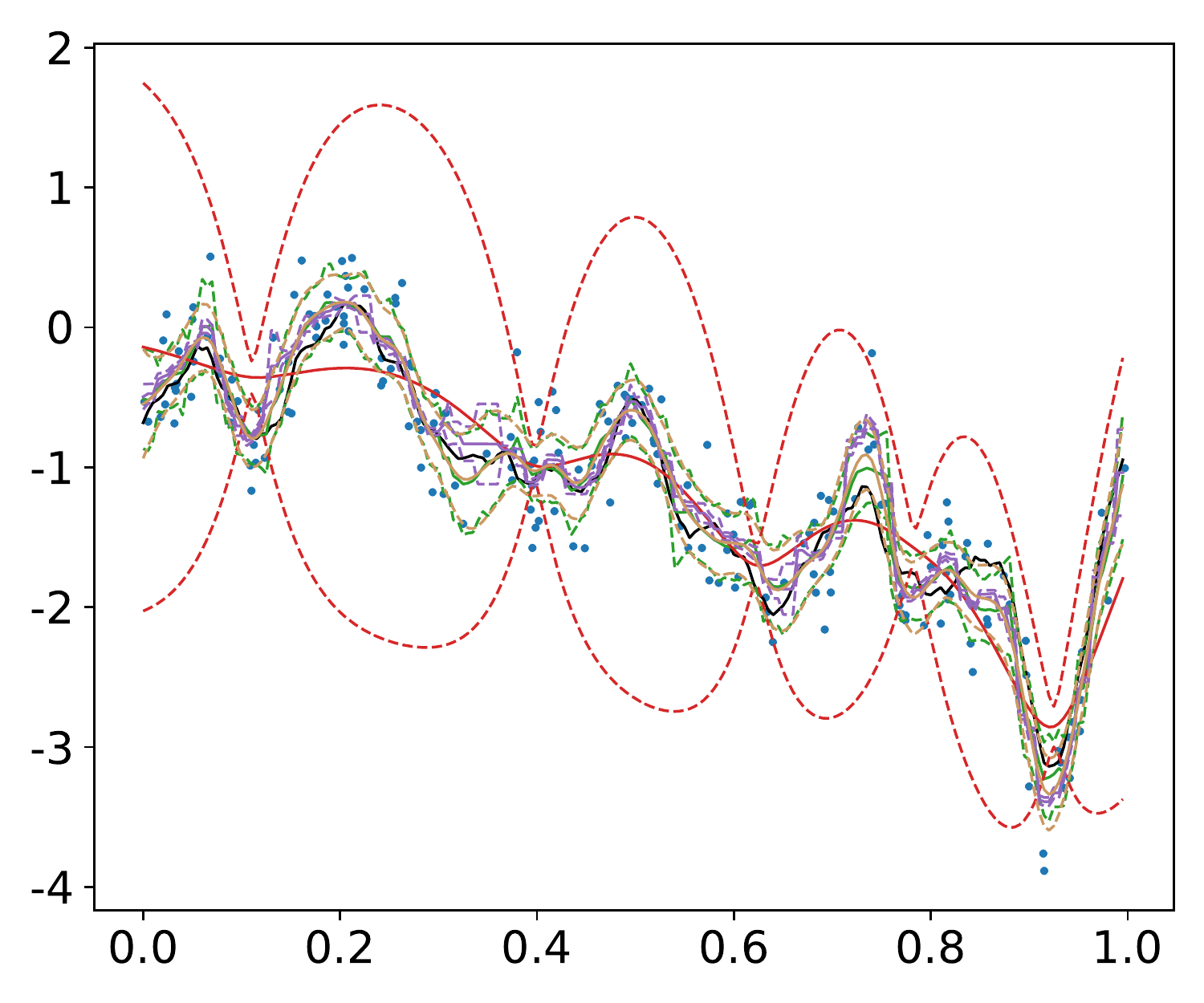}
	\end{subfigure}%
    \hfill
    \begin{subfigure}{.48\textwidth}
	\centering
 	\includegraphics[trim={0 0cm 0 0}, clip, width =.99\linewidth]{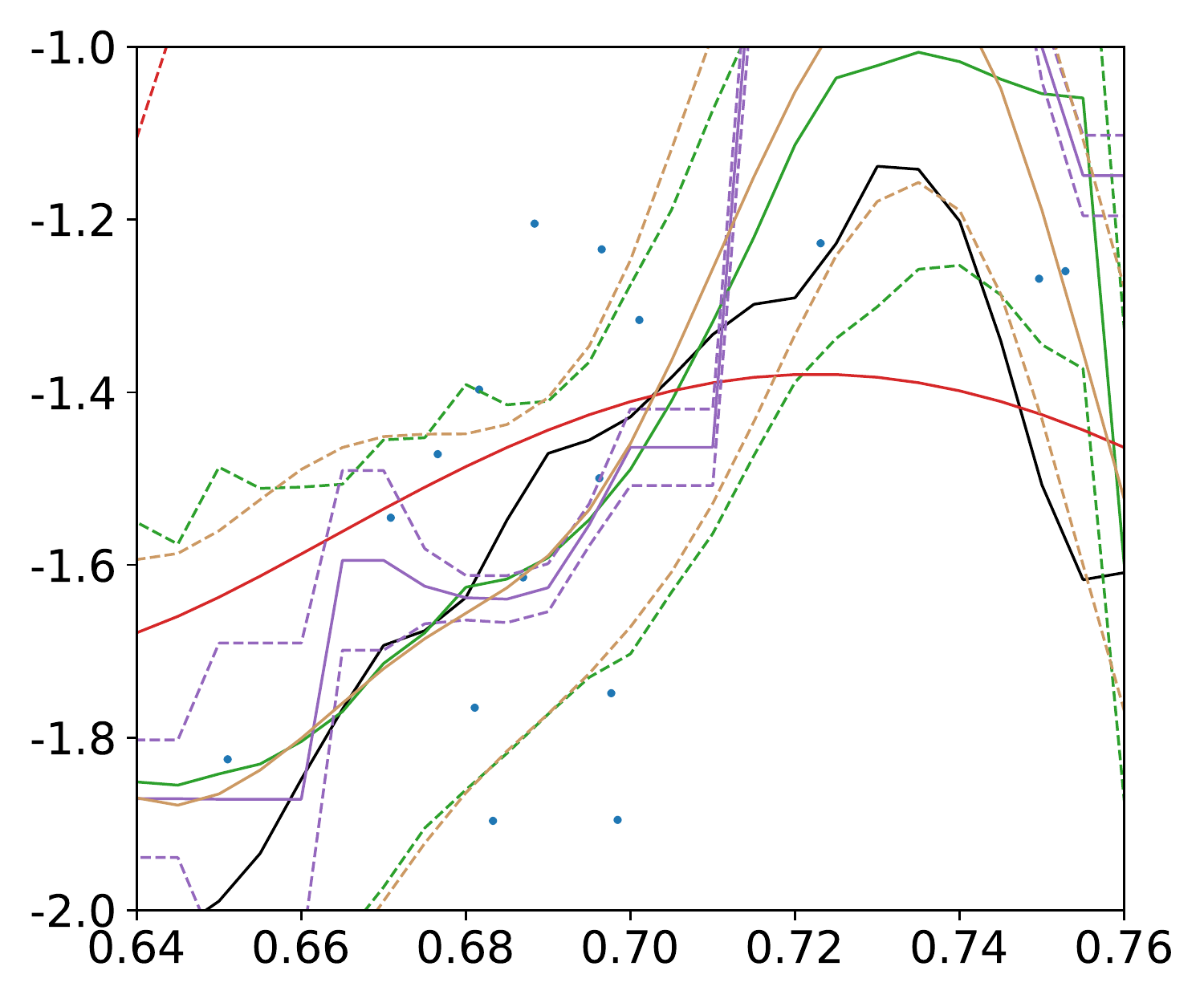}
	\end{subfigure}%
    \hfill
    
   \begin{subfigure}{0.6\textwidth}
	\centering
 	\includegraphics[width =.99\linewidth]{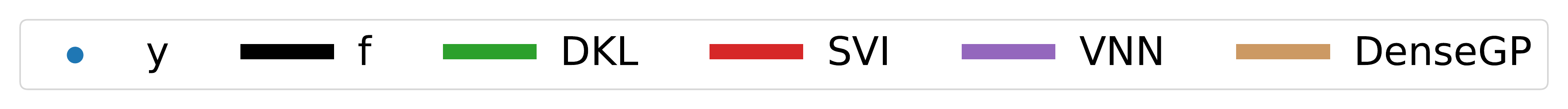}
	\end{subfigure}%
	
  \caption{Comparison of exact GP predictions (DenseGP) to three variational GP approximations for simulated data with Gaussian noise at $n=200$ randomly sampled training inputs on $[0,1]$ with $\sigma_{\epsilon} = 0.3$ and true range $\lambda=0.1$. We show the means (solid lines) and $95\%$ pointwise intervals of the posterior predictive distribution $\mathbf{f}^* | \by$ at 200 regularly spaced test inputs. The right panel zooms into a smaller region of the left panel to highlight the differences.}
\label{fig:1D_infer}
\end{figure*}

Figure~\ref{fig:1D_infer} provides a visual comparison of SVIGP, VNNGP, and DKLGP predictions for a toy example in one dimension. We also included predictions from the exact GP (DenseGP) which cannot be obtained for large $n$. DKLGP approximated the DenseGP most closely, especially in terms of the prediction intervals. SVIGP oversmoothed heavily and produced very wide prediction intervals. VNNGP assumes a diagonal covariance in the variational distribution $q(\mathbf{f})$, which appears to have caused sharply fluctuating predictions and narrow prediction intervals.
Figure~\ref{fig:1D_infer_part2} in Appendix~\ref{app:additional} shows similar comparisons for Student-$t$ and Bernoulli likelihoods.

\subsection{Results on Synthetic Data}
\label{subsec:sim_study}

We also carried out a more comprehensive comparison for $10{,}000$ inputs randomly distributed in the unit hypercube, $[0,1]^5$, with true range parameters $\bflambda = (0.25, 0.50, 0.75, 1.00, 1.25)$. We used $n=8{,}000$ inputs for training and $2{,}000$ for testing. Performance was measured in terms of the variational inference of the latent field $f(\cdot)$ at training and test inputs. For each scenario, results over five replicates were produced and averaged.
\begin{figure*}[htbp]
\centering
	\begin{subfigure}{.99\textwidth}
	\centering
 	\includegraphics[width =.99\linewidth]{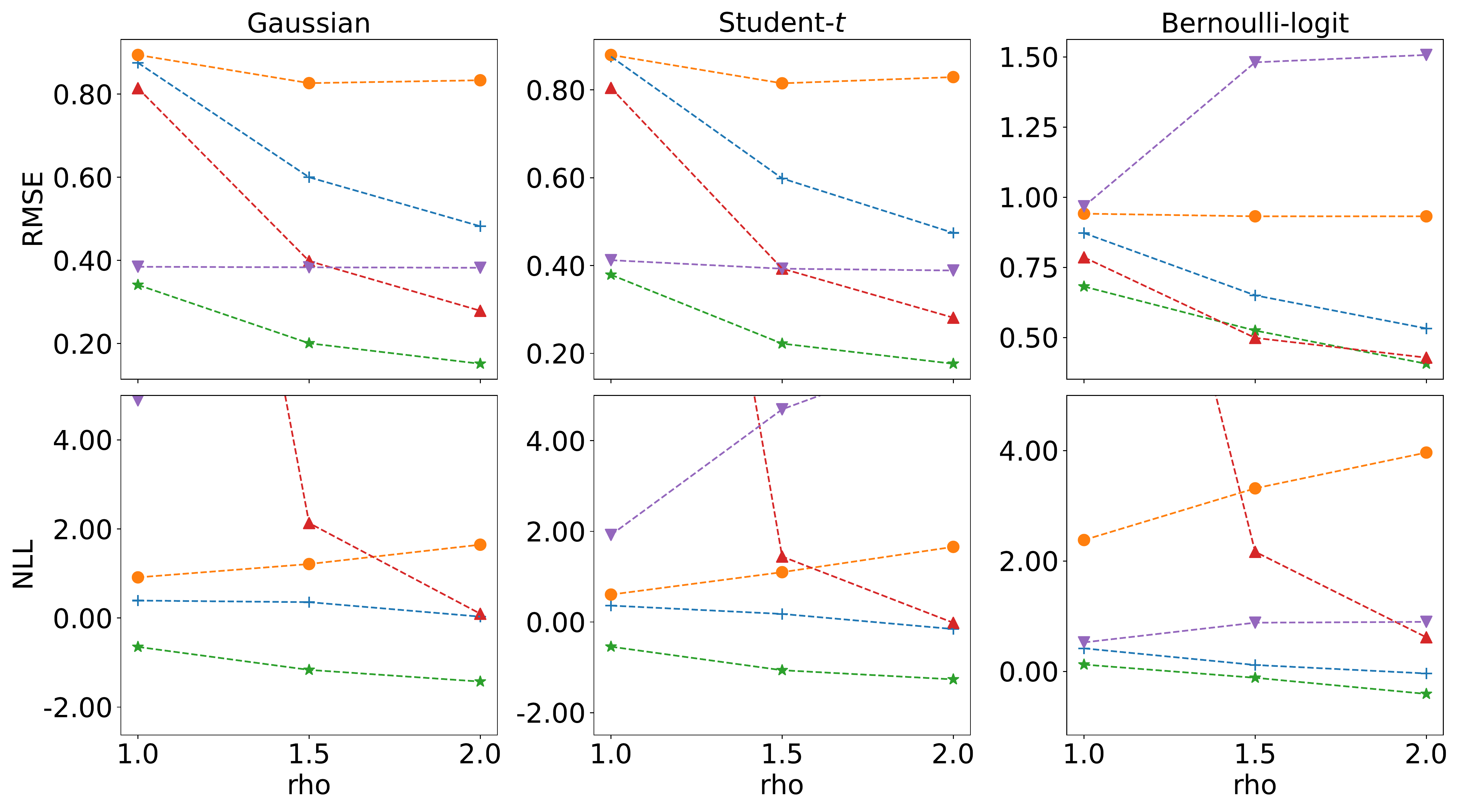}
	\end{subfigure}%

   \begin{subfigure}{0.6\textwidth}
	\centering
 	\includegraphics[width =.99\linewidth]{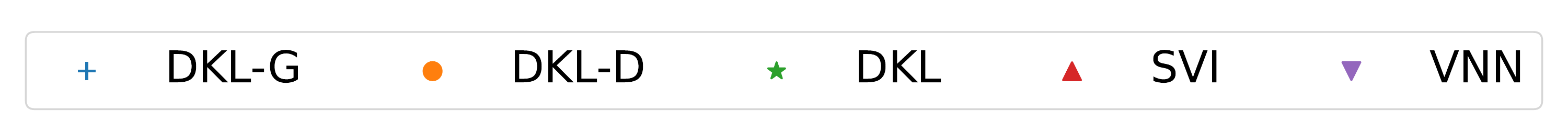}
	\end{subfigure}%
	
  \caption{RMSE (top) and NLL (bottom) for predicting the latent field at test inputs for simulated data in a five-dimensional input domain, as a function of the complexity parameter $\rho$, with Gaussian (left), Student-$t$ (center) and Bernoulli-logit (right) likelihoods}
\label{fig:outsample_f}
\end{figure*}

\begin{table*}[htb]
    \caption{RMSE and NLL at held-out test points averaged over five splits for several UCI datasets, ordered from low to high dimension $d$. The Student-$t$ and Bernoulli-logit likelihoods were used for PRECIP and COVTYPE, respectively; a Gaussian likelihood was used for the other datasets. The average sparsity-set size for DKL is denoted by $m$. SVI used $m$ inducing points, while SVI$_{32}$ and SVI$_{512}$ used 32 and 512 points, respectively. While SVI$_{32}$ and SVI$_{512}$ are included for reference, they exhibit substantially higher computational complexity and training time than the other approaches and are hence colored in grey.}
    \label{tbl:real_data}
    \vspace{2mm}
    \centering
    \begin{small}
    \begin{sc}
    \setlength\tabcolsep{4.5pt}
    \begin{tabular}{c|r|r|r|r|r|r|r|r|r|r|r|r|r|r|r|r|r|r}
    & \multicolumn{2}{c|}{3DRoad} & \multicolumn{2}{c|}{Precip} & \multicolumn{2}{c|}{Kin40k} & \multicolumn{2}{c|}{Protein} & \multicolumn{2}{c|}{Bike} & \multicolumn{2}{c|}{elevators} & \multicolumn{2}{c|}{KEGG} & \multicolumn{2}{c|}{KEGGU} & \multicolumn{2}{c}{Covtype}\\
    $n$, $d$ & \multicolumn{2}{c|}{65K, 3} & \multicolumn{2}{c|}{85K, 3} & \multicolumn{2}{c|}{40K, 8}& \multicolumn{2}{c|}{44K, 9} & \multicolumn{2}{c|}{17K, 17 }& \multicolumn{2}{c|}{17K, 18} & \multicolumn{2}{c|}{16K, 20} & \multicolumn{2}{c|}{18K, 26} & \multicolumn{2}{c}{100K, 53} \\
    $m$ &  \multicolumn{2}{c|}{2} & \multicolumn{2}{c|}{5} & \multicolumn{2}{c|}{7} & \multicolumn{2}{c|}{8} & \multicolumn{2}{c|}{12} & \multicolumn{2}{c|}{22} & \multicolumn{2}{c|}{19} & \multicolumn{2}{c|}{21} & \multicolumn{2}{c}{3}\\
    \hline
    SVI & .80& .28 & .91& .43 & .61& .01 & .81& 0.29 & .09& -1.85 & .39& -.43 & .08& -2.05 & .06& -2.30 & .50& NA\\
    \textcolor{gray}{SVI$_{32}$} & \textcolor{gray}{.59}& \textcolor{gray}{-.02} & \textcolor{gray}{.83}& \textcolor{gray}{.34} & \textcolor{gray}{.37}& \textcolor{gray}{-.47} & \textcolor{gray}{.75}& \textcolor{gray}{0.22} & \textcolor{gray}{.06} & \textcolor{gray}{-2.21} & \textcolor{gray}{.39} & \textcolor{gray}{-.45} & \textcolor{gray}{.07} & \textcolor{gray}{-2.16} & \textcolor{gray}{\textbf{.06}}& \textcolor{gray}{\textbf{-2.29}} & \textcolor{gray}{.50} & \textcolor{gray}{NA}\\
    \textcolor{gray}{SVI$_{512}$} & \textcolor{gray}{.38}& \textcolor{gray}{-.44} & \textcolor{gray}{.64} & \textcolor{gray}{.11} & \textcolor{gray}{\textbf{.17}}& \textcolor{gray}{\textbf{-1.2}} & \textcolor{gray}{.67} & \textcolor{gray}{0.10} & \textcolor{gray}{\textbf{.03}} & \textcolor{gray}{\textbf{-2.69}} & \textcolor{gray}{\textbf{.37}} & \textcolor{gray}{\textbf{-.49}} & \textcolor{gray}{\textbf{.07}} & \textcolor{gray}{\textbf{-2.22}} & \textcolor{gray}{.06} & \textcolor{gray}{-2.28} & \textcolor{gray}{.50} & \textcolor{gray}{NA} \\
    VNN & .28& 2.16 & .49& 4.35 & .56& 24.19 & .69& 5.59 & .49& 7.60 & .65& 1.26 & .13& 1.22 & .14& 3.85 & NA& NA\\
    DKL & \textbf{.27}& \textbf{-.83} & \textbf{.41}& \textbf{-.38} & .37& -.55 & \textbf{.56}& \textbf{-.19} & .11& -1.63 & .43& -.37 & .09& -1.97 & .11& -2.08 & \textbf{.28}& NA
    \end{tabular}
    \end{sc}
    \end{small}
\end{table*}

\begin{table*}[htb]
    \caption{Comparison of wall-clock time (in seconds) for the datasets and methods in Table~\ref{tbl:real_data}. $\sparsity \& \anc$ refers to computing the r-maximin ordering, sparsity pattern and ancestor sets.}
    \label{tbl:real_data_timing}
    \vspace{2mm}
    \centering
    \begin{small}
    \begin{sc}
    \setlength\tabcolsep{4.5pt}
    \begin{tabular}{c|r|r|r|r|r|r|r|r|r}
    & \multicolumn{1}{c|}{3DRoad} & \multicolumn{1}{c|}{Precip} & \multicolumn{1}{c|}{Kin40k} & \multicolumn{1}{c|}{Protein} & \multicolumn{1}{c|}{Bike} & \multicolumn{1}{c|}{elevators} & \multicolumn{1}{c|}{KEGG} & \multicolumn{1}{c|}{KEGGU} & \multicolumn{1}{c}{Covtype}\\
    \hline
    SVI & 3{,}283   & 3{,}965   & 3{,}305   & 3{,}589    & 1{,}487 & 1{,}386      & 1{,}329 & 1{,}594  & 5{,}945    \\
    \textcolor{gray}{SVI$_{32}$} & \textcolor{gray}{8{,}879}   & \textcolor{gray}{9{,}207}   & \textcolor{gray}{5{,}460}   & \textcolor{gray}{5{,}941}    & \textcolor{gray}{3{,}082} & \textcolor{gray}{2{,}952}      & \textcolor{gray}{3{,}159} & \textcolor{gray}{3{,}387}  & \textcolor{gray}{7{,}722}    \\
    \textcolor{gray}{SVI$_{512}$} & \textcolor{gray}{25{,}409} & \textcolor{gray}{26{,}223} & \textcolor{gray}{21{,}710} & \textcolor{gray}{22{,}179} & \textcolor{gray}{12{,}232} & \textcolor{gray}{9{,}518} & \textcolor{gray}{9{,}988} & \textcolor{gray}{10{,}642} & \textcolor{gray}{44{,}839} \\
    VNN & 2{,}788   & 3{,}332   & 1{,}696   & 2{,}081    & 568  & 454      & 487  & 595   & NA      \\
    DKL & 1{,}591   & 3{,}948   & 1{,}859   & 2{,}736    & 3{,}129 & 807      & 1{,}285 & 1{,}536  & 4{,}932   \\
    $\sparsity \& \anc$ & 90   & 268   & 2{,}866   & 440    & 170 & 577      & 171 & 207  & 1{,}277  
    \end{tabular}
    \end{sc}
    \end{small}
\end{table*}

Figure~\ref{fig:outsample_f} compares root mean squared error (RMSE) and negative log-likelihood (NLL) at test inputs. For the Gaussian and Student-$t$ likelihoods, DKLGP produced the most accurate predictions, while for the Bernoulli-logit likelihood, SVIGP and DKLGP appeared to be similarly accurate in terms of RMSE. DKLGP outperformed the competing methods in terms of NLL. While DKLGP, DKL-G, and SVIGP improved with increasing $\rho$ as expected, the mean-field approximations (VNNGP and DKL-D) generally did not. We performed the same comparison for the squared exponential and rational quadratic kernels in Figures~\ref{fig:outsample_f_RBF} and \ref{fig:outsample_f_RQ} in Appendix~\ref{app:additional}, which resulted in the same rankings as for the Mat\'ern kernel, except that DKLGP was marginally outperformed by SVIGP in terms of RMSE for the Bernoulli-logit likelihood at $\rho = 2.0$.

We also computed RMSE and NLL scores at training inputs for the methods we considered in Figure~\ref{fig:outsample_f}, as presented in Figure~\ref{fig:insample_f} in Appendix \ref{app:additional}. Consistent with the results from Figure~\ref{fig:outsample_f}, DKLGP generally performed best, which is consistent with the results from Figure~\ref{fig:outsample_f}. VNNGP performed similarly to DKLGP for the Gaussian and Student-$t$ likelihoods, but underestimated the variance at test inputs and so led to poor NLL scores. Note that variational methods are generally known to underestimate the posterior variance \citep{blei2017variational}.

\subsection{Results on UCI Data}

To provide a more comprehensive comparison of SVIGP, VNNGP and DKLGP, we considered datasets from the UCI data repository widely used for benchmarking purposes. For the UCI datasets we considered in this section, covariates were first standardized to $[0, 1]$ and removed from analysis if the standard deviation after standardization was smaller than $0.01$. Furthermore, inputs were filtered to ensure that the minimum distance between inputs was greater than $0.001$ to prevent numerical singularity. Approximately $20\%$ of each dataset was used for testing. We chose different $\rho$ for different datasets and computed the corresponding $m$. Since Section~\ref{subsec:sim_study} demonstrated the advantage of DKLGP over its variants DKL-G and DKL-D, we excluded the two variants here for ease of presentation. We included SVIGP with $m = 32$ and $m = 512$ inducing points as benchmarks for easier comparison with relevant works in the literature. 

Table~\ref{tbl:real_data} summarizes the performance of the three methods across nine UCI datasets. DKLGP had better scores than VNNGP for all datasets except for COVTYPE, for which VNNGP ran out of memory on a 64GB node despite having reduced the data size to a subset of size $100$K. Compared to the SVIGP with similar computation cost, DKLGP provided substantially better performance for the binary response data COVTYPE and for low-dimensional ($d<10$) settings, and roughly similar performance for most high-dimensional datasets except the KEGGU data, for which SVIGP produced much lower RMSE than DKLGP. However, this does not appear to be due to DKLGP providing a less accurate approximation to the exact GP, but rather it appears to be due to the exact GP (with its simple ARD kernel) being severely misspecified for KEGGU. To explore this further, we fitted the exact GP (DenseGP) to KEGGU. The DenseGP's RMSE was 0.14 (same as for DKLGP), and the root average squared distance between the DenseGP predictions and the DKLGP and SVIGP predictions was $0.05$ and $0.13$, respectively, which implies that the DKLGP predictions were a much better approximation of the exact-GP predictions than the SVIGP predictions. VNNGP provided better point predictions than SVIGP for the low-dimensional datasets, which is consistent with the results in \citet{Wu2022}; however, VNNGP's NLL was high due to its underestimation of posterior variance. 

Table~\ref{tbl:real_data_timing} summarizes the wall-clock times on an Intel Xeon E5-2680 v4 CPU with 14 cores and 28 threads for the methods under comparison, where the computation of sparsity and ancestor sets is only applicable to DKLGP or DKL. The DKL computation times were closer to those of SVI than to those of SVI$_{32}$, indicating that DKLGP and SVIGP should be compared at the same $m$ on the basis of comparable computation times. While increasing $m$ significantly improved SVIGP's performance on low-dimensional datasets, even $m=512$ inducing points made the training of SVIGP challenging on the workstation we used for comparison. The performance of DKLGP can also be improved by using a larger $\rho$; for example, DKLGP's RMSE for the KIN40K data was reduced to $0.27$ for $m = 21$.

\section{Conclusion \label{sec:conclusions}}

We have introduced a variational approach using a variational family and approximate prior based on SIC restrictions. The (r-)maximin ordering, nearest-neighbor sparsity pattern, and a computational trick based on reduced ancestor sets together result in efficient and accurate inference and prediction for LGPs. 
While the time complexity is cubic in the number of neighbors, quadratic complexity for the prior approximation can be achieved by grouping observations and re-using Cholesky factors \citep{Schafer2020}; we will investigate an extension of this idea to computing the ELBO in our variational setting.
Although we here assume that the input domain is Euclidean, our method can be applied more generally; using a correlation-based distance instead of Euclidean distance \citep{Kang2021}, one can use our method to perform LGP inference for large data on complex domains \citep[cf.][]{tibo2022inducing}. We will also explore extensions to deep GPs \citep[cf.][]{Sauer2022}. 
An implementation of our method, along with code to reproduce all results, is publicly available at \url{https://github.com/katzfuss-group/DKL-GP}.

Our approach is applicable to irregularly spaced observations and in principle to any desired covariance structure. Our method provides state-of-the-art performance when fine-scale structure in the function of interest can be discerned from the data; in contrast, if the data are highly noisy or sparse or the covariance model is severely misspecified, inducing-point methods such as SVIGP that produce smooth predictions and wide uncertainty intervals may be competitive with our approach.

\section*{Acknowledgments}

Jian Cao was partially supported by the Texas A\&M Institute of Data Science (TAMIDS) Postdoctoral Project program, Jian Cao and Matthias Katzfuss by National Science Foundation (NSF) Grant DMS--1654083, and Felix Jimenez and Matthias Katzfuss by NSF Grant DMS--1953005. We would like to thank Luhuan Wu for helpful comments and discussions.

\bibliography{refs_compact}
\bibliographystyle{icml2023}

\newpage
\appendix
\onecolumn
\section{Additional numerical results\label{app:additional}}

This section contains additional figures not shown in the main paper. Complementing Figure~\ref{fig:sizes_n_rho}, Figure~\ref{fig:sizes_rho} shows that reduced ancestor sets $\tilde{\anc}_i$ are much smaller than full ancestor sets $\anc_i$ across a range of $\rho$ values.

\begin{figure*}[ht]
\centering
\hfill
\begin{subfigure}{.33\textwidth}
\centering
\includegraphics[width =.99\linewidth]{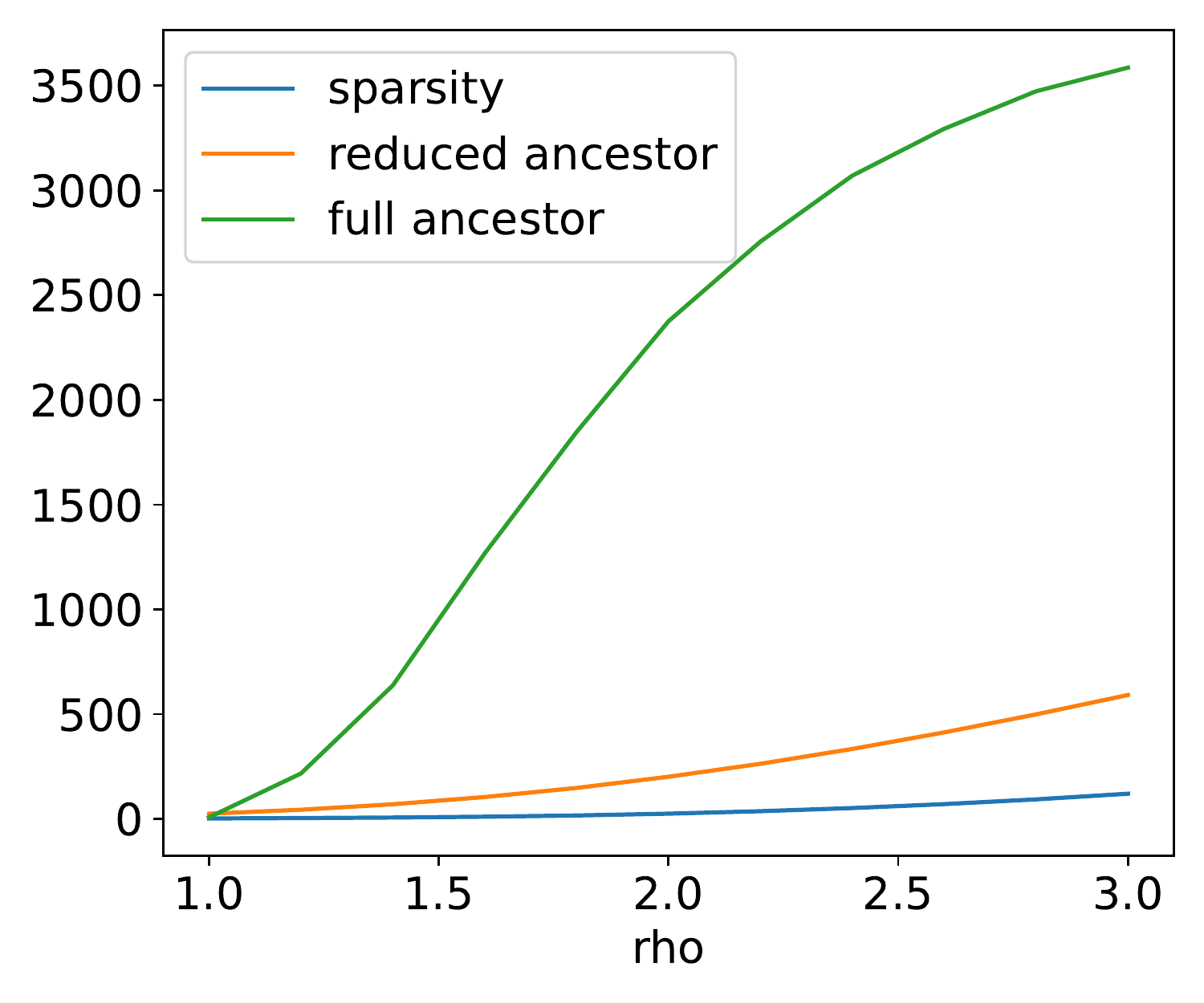}
\end{subfigure}
\hfill ~
\caption{Average sizes of the sparsity sets $\sparsity_i$, reduced ancestor sets $\tilde\anc_i$, and full ancestor sets $\anc_i$ as a function of $\rho$ with $n=8{,}000$. The inputs are sampled uniformly on $[0,1]^5$.}
\label{fig:sizes_rho}
\end{figure*}

Complementing Figure~\ref{fig:effectofreduced}, Figures~\ref{fig:effectofreduced_RBF} and \ref{fig:effectofreduced_RQ} show that the approximation error in computing the ELBO caused by using reduced ancestor sets is negligible even for the squared-exponential and rational-quadratic kernels, respectively.
\begin{figure*}[h!]
\centering
\begin{subfigure}{.33\textwidth}
\centering
\includegraphics[width =.99\linewidth]{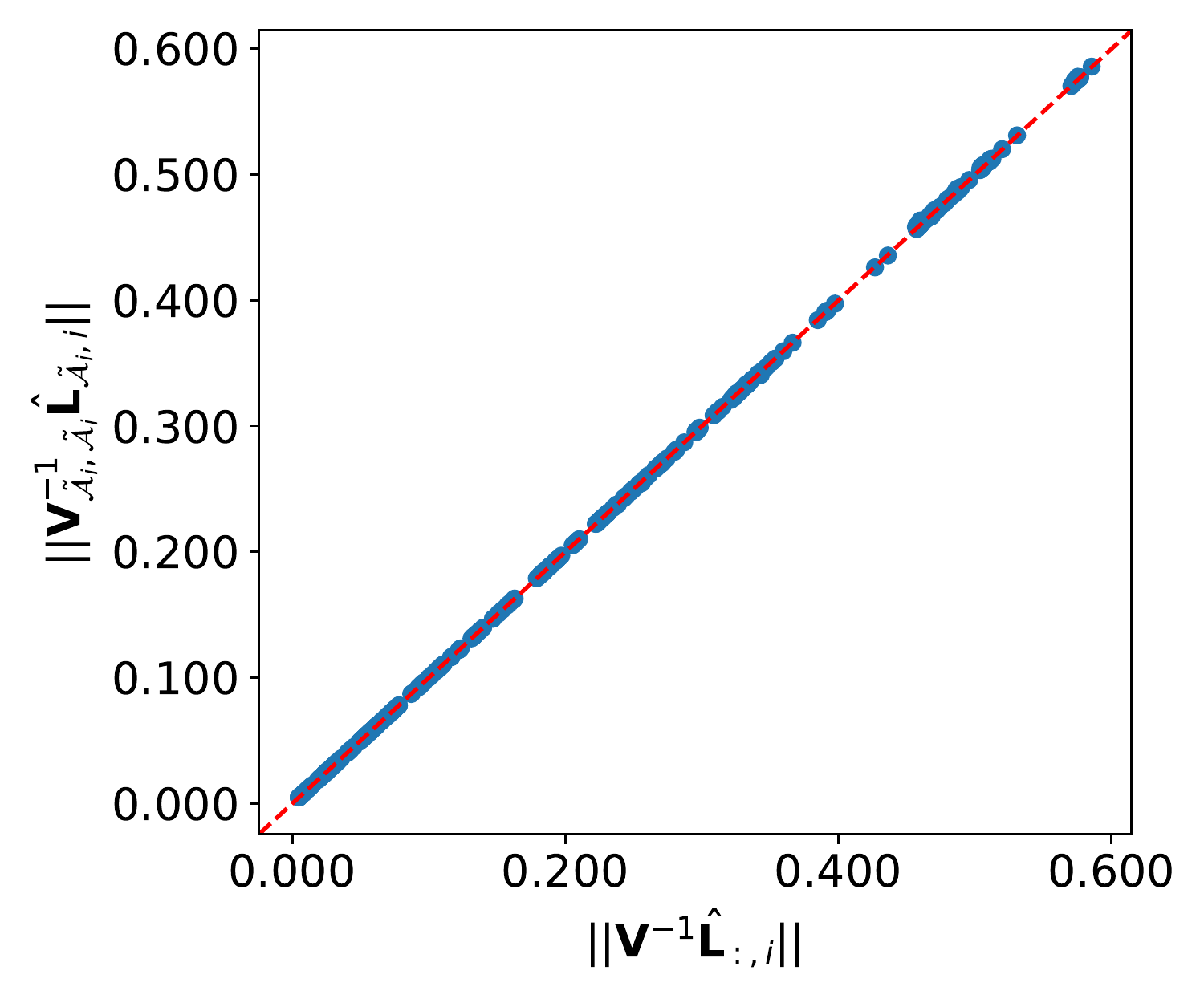}
\caption{$\|\bV_{\tilde\anc_i,\tilde\anc_i}^{-1}\hat\bL_{\tilde\anc_i,i}\|$ vs $\|\bV^{-1}\hat\bL_{:,i}\|$}
\end{subfigure}
\begin{subfigure}{.33\textwidth}
\centering
\includegraphics[width =.99\linewidth]{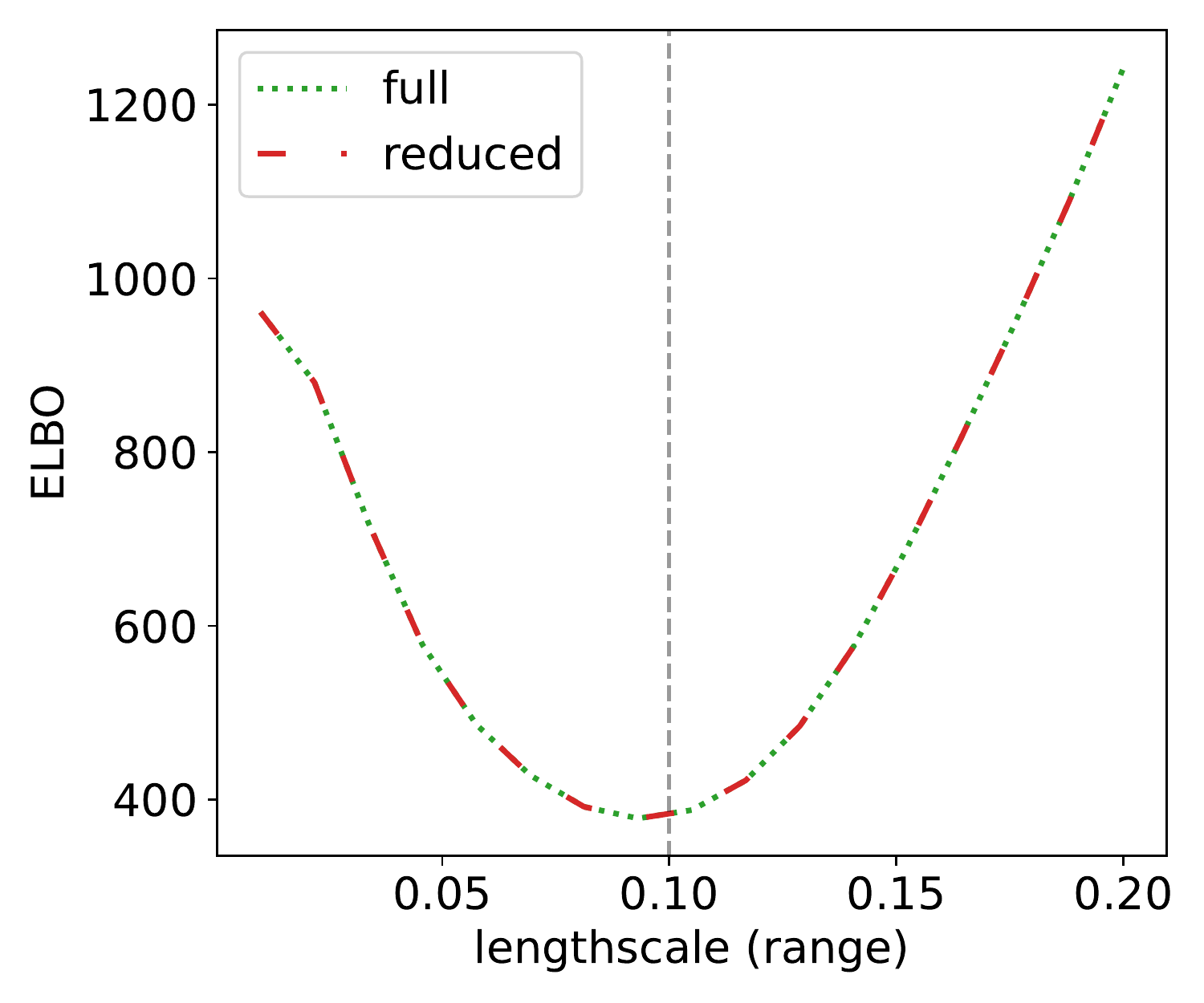}
\caption{ELBO using reduced ancestors}
\end{subfigure}
\caption{The squared-exponential-kernel versions of Figures \ref{fig:reducedsolve} (left) and \ref{fig:ELBO_err} (right): The left figure compares $\|\bV_{\tilde\anc_i,\tilde\anc_i}^{-1}\hat\bL_{\tilde\anc_i,i}\|$ with reduced ancestor sets versus $\|\bV^{-1}\hat\bL_{:,i}\|$  for $i = 1, \ldots, n$, where $n = 500$ and $d = 2$. The right figure compares ELBO curves based on full \eqref{eq:ELBOhat} and reduced \eqref{eq:ELBOreduced} ancestor sets, as functions of the range parameter with true value $0.1$, for $n=500$ and $d=2$. In all plots, we set $\rho=2$ and the $n$ inputs are sampled uniformly on $[0,1]^d$.}
\label{fig:effectofreduced_RBF}
\end{figure*}

\begin{figure*}[h!]
\centering
\begin{subfigure}{.33\textwidth}
\centering
\includegraphics[width =.99\linewidth]{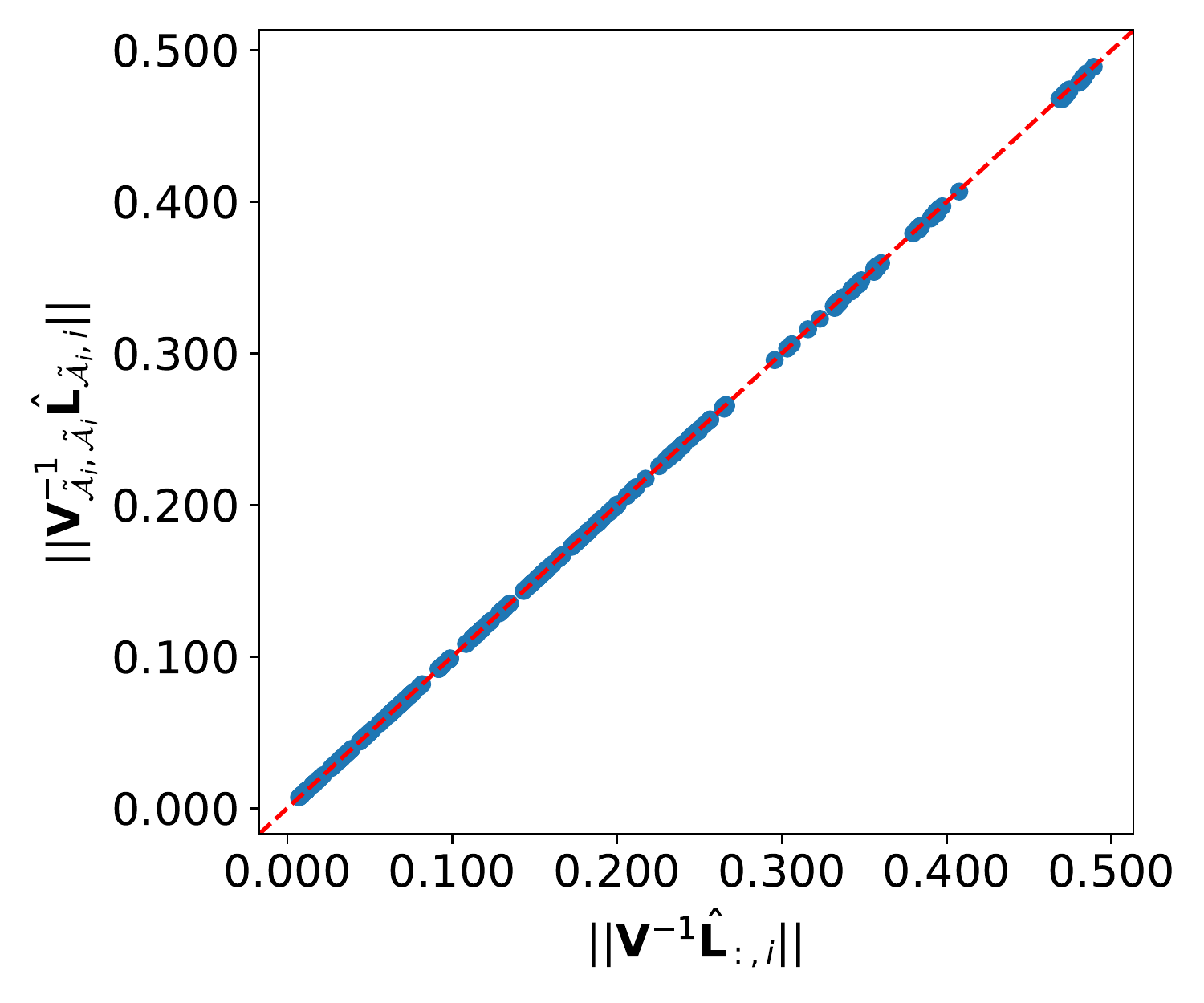}
\caption{$\|\bV_{\tilde\anc_i,\tilde\anc_i}^{-1}\hat\bL_{\tilde\anc_i,i}\|$ vs $\|\bV^{-1}\hat\bL_{:,i}\|$}
\end{subfigure}
\begin{subfigure}{.33\textwidth}
\centering
\includegraphics[width =.99\linewidth]{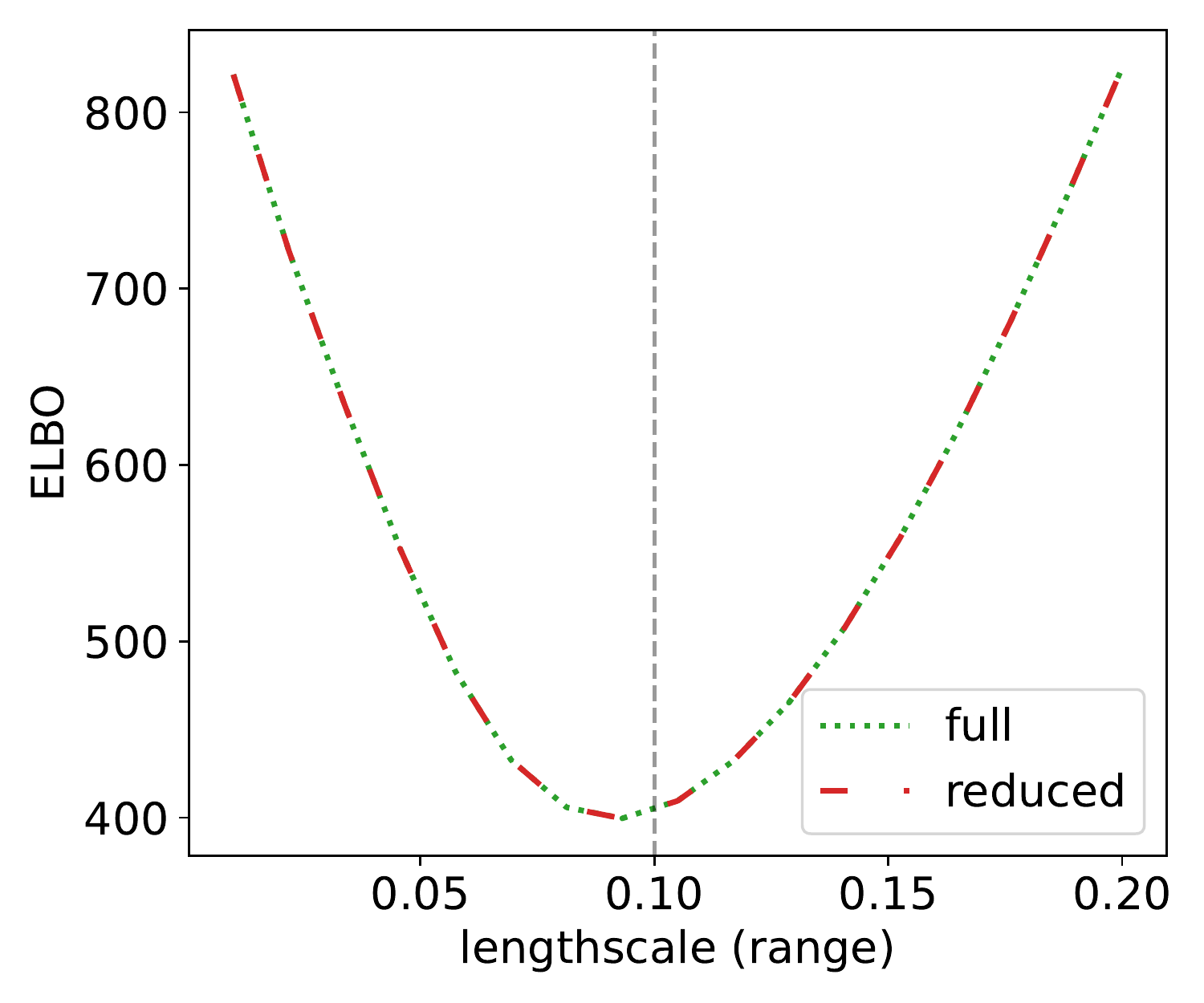}
\caption{ELBO using reduced ancestors}
\end{subfigure}
\caption{The rational-quadratic-kernel versions of Figures \ref{fig:reducedsolve} (left) and \ref{fig:ELBO_err} (right): The left figure compares $\|\bV_{\tilde\anc_i,\tilde\anc_i}^{-1}\hat\bL_{\tilde\anc_i,i}\|$ with reduced ancestor sets versus $\|\bV^{-1}\hat\bL_{:,i}\|$  for $i = 1, \ldots, n$, where $n = 500$ and $d = 2$. The right figure compares ELBO curves based on full \eqref{eq:ELBOhat} and reduced \eqref{eq:ELBOreduced} ancestor sets, as functions of the range parameter with true value $0.1$, for $n=500$ and $d=2$. In all plots, we set $\rho=2$ and the $n$ inputs are sampled uniformly on $[0,1]^d$.}
\label{fig:effectofreduced_RQ}
\end{figure*}

Figure~\ref{fig:post_mean_opt} suggests that the initialization of $\bfnu$ using Vecchia-Laplace approximation and incomplete Cholesky (IC0) approximation provides reasonable starting values for $\bfnu$, which can be further refined by optimizing the ELBO.

\begin{figure}[h!]
\centering
\includegraphics[width =.4\linewidth]{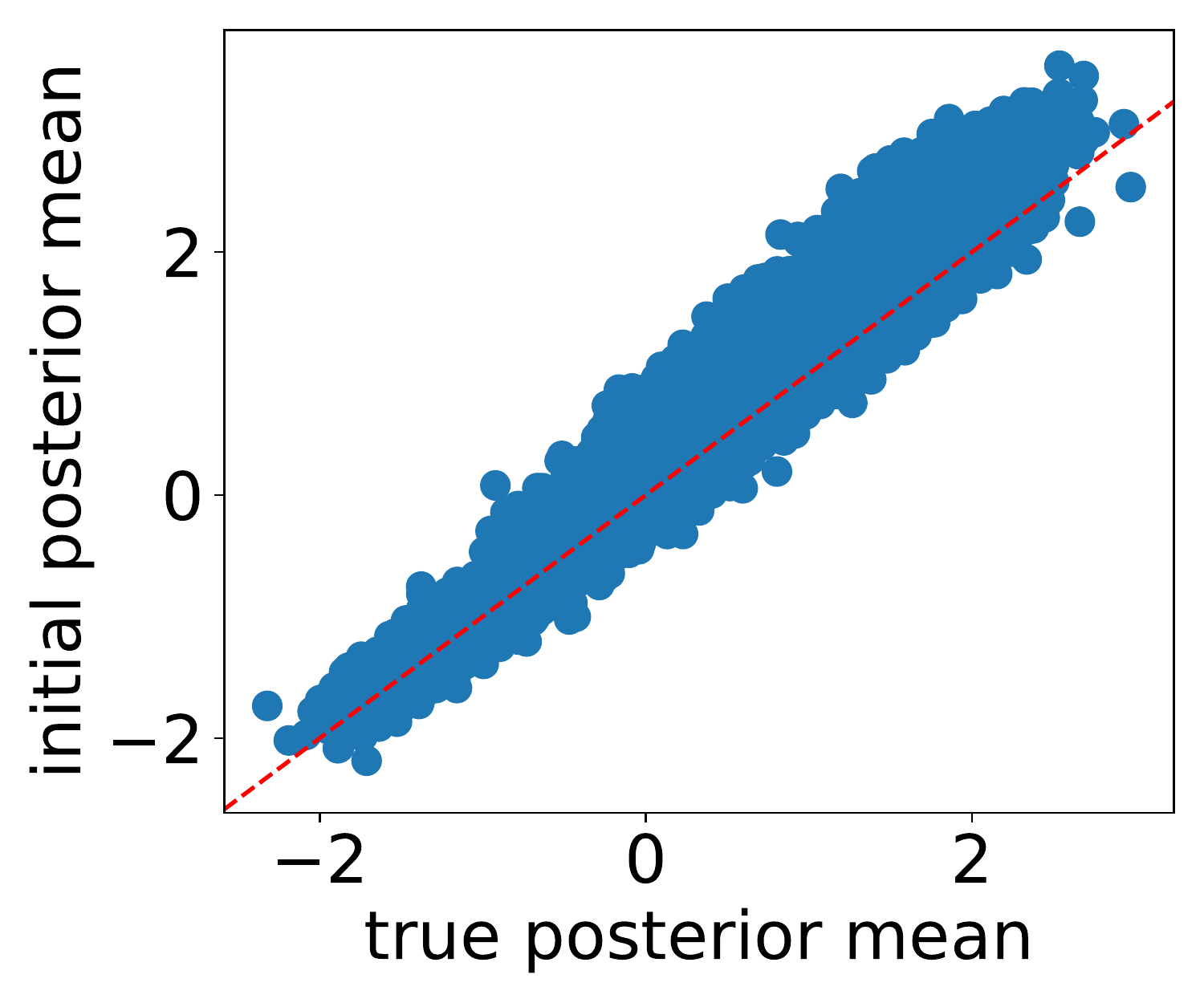}
\includegraphics[width =.4\linewidth]{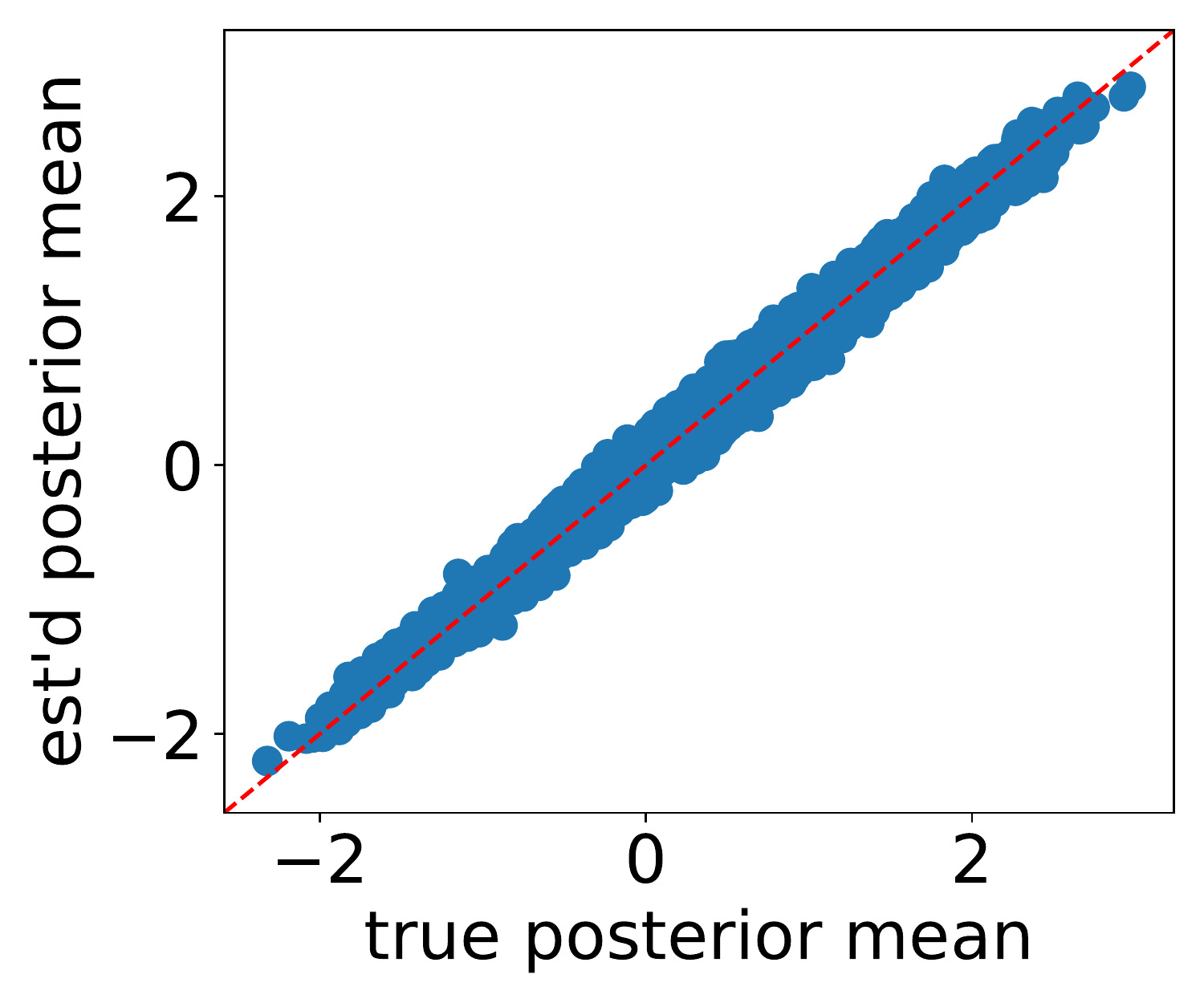}
\caption{Comparison of estimated and true posterior means: The left panel uses the estimated posterior means at the initialization step (i.e., the IC0 solution) and right panel used the estimated posterior means after ELBO optimization for simulated Gaussian data. The simulation setting is the same as in Figure~\ref{fig:insample_f}.}
\label{fig:post_mean_opt}
\end{figure}

Figure~\ref{fig:insample_f} shows a comparison of RMSE and NLL scores for the posterior marginals of the entries of $\mathbf{f}$ at training inputs. In contrast to Figure~\ref{fig:outsample_f}, VNNGP performed similarly to DKLGP and outperformed SVIGP for Gaussian and Student-t likelihoods. Furthermore, the Vecchia-Laplace approximation with IC0 (used as the initialization for DKLGP) was usually the third best model, indicating an advantage of using the SIC restriction for $\bL$ and $\bV$.
\begin{figure*}[h!]
\centering
	\begin{subfigure}{.99\textwidth}
	\centering
 	\includegraphics[width =.99\linewidth]{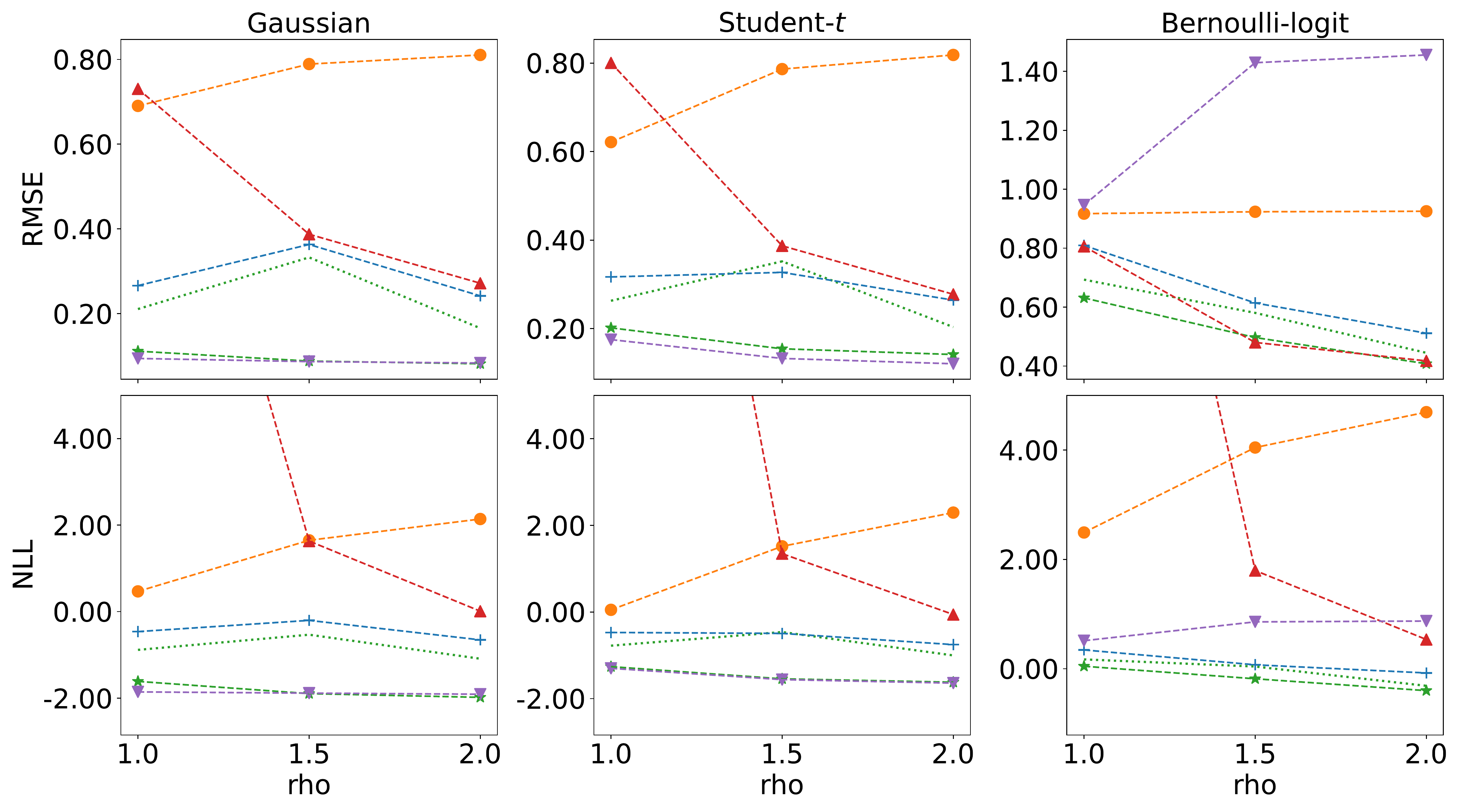}
	\end{subfigure}%

   \begin{subfigure}{0.65\textwidth}
	\centering
 	\includegraphics[width =.99\linewidth]{plots/legend.pdf}
	\end{subfigure}%
 
  \caption{RMSE (top) and NLL (bottom) for predicting the latent field at training inputs, as a function of the complexity parameter $\rho$, with Gaussian (left), Student-$t$ (center) and Bernoulli-logit (right) likelihoods, under the same experimental setting in Figure~\ref{fig:outsample_f}. The green dotted lines are the scores of the model obtained only by initialization using Vecchia-Laplace approximation and IC0.}
\label{fig:insample_f}
\end{figure*}

Figure \ref{fig:1D_infer_part2} shows one-dimensional toy examples for Student-$t$ and Bernoulli-logit likelihoods. Note that the DenseGP is available only for the Gaussian likelihood.

\begin{figure*}[h!]
\centering
\hfill
	\begin{subfigure}{.45\textwidth}
	\centering
 	\includegraphics[trim={0 0cm 0 0}, clip, width =.99\linewidth]{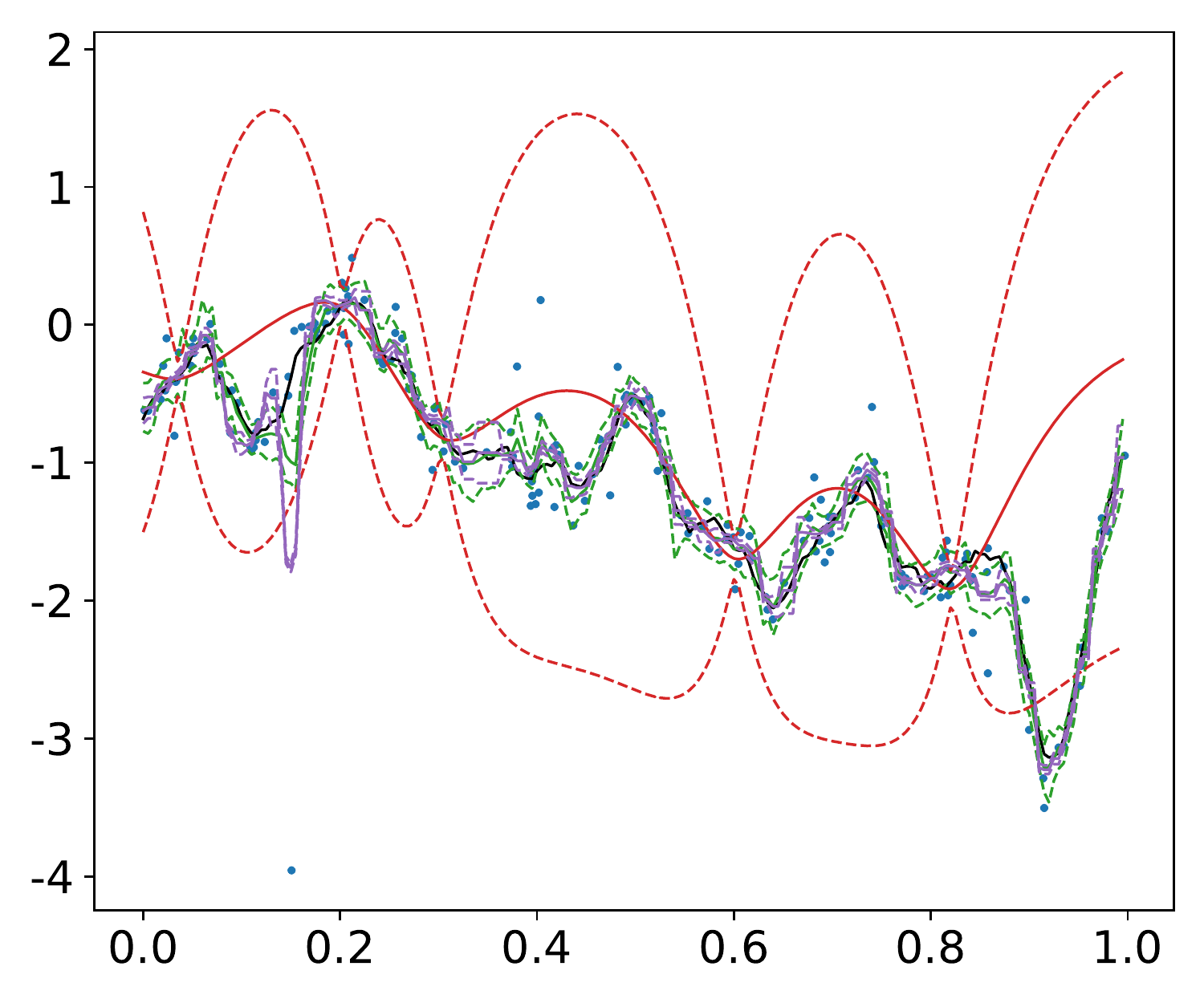}
	\end{subfigure}%
\hfill
	\begin{subfigure}{.45\textwidth}
	\centering
 	\includegraphics[trim={0 0cm 0 0}, clip, width =.99\linewidth]{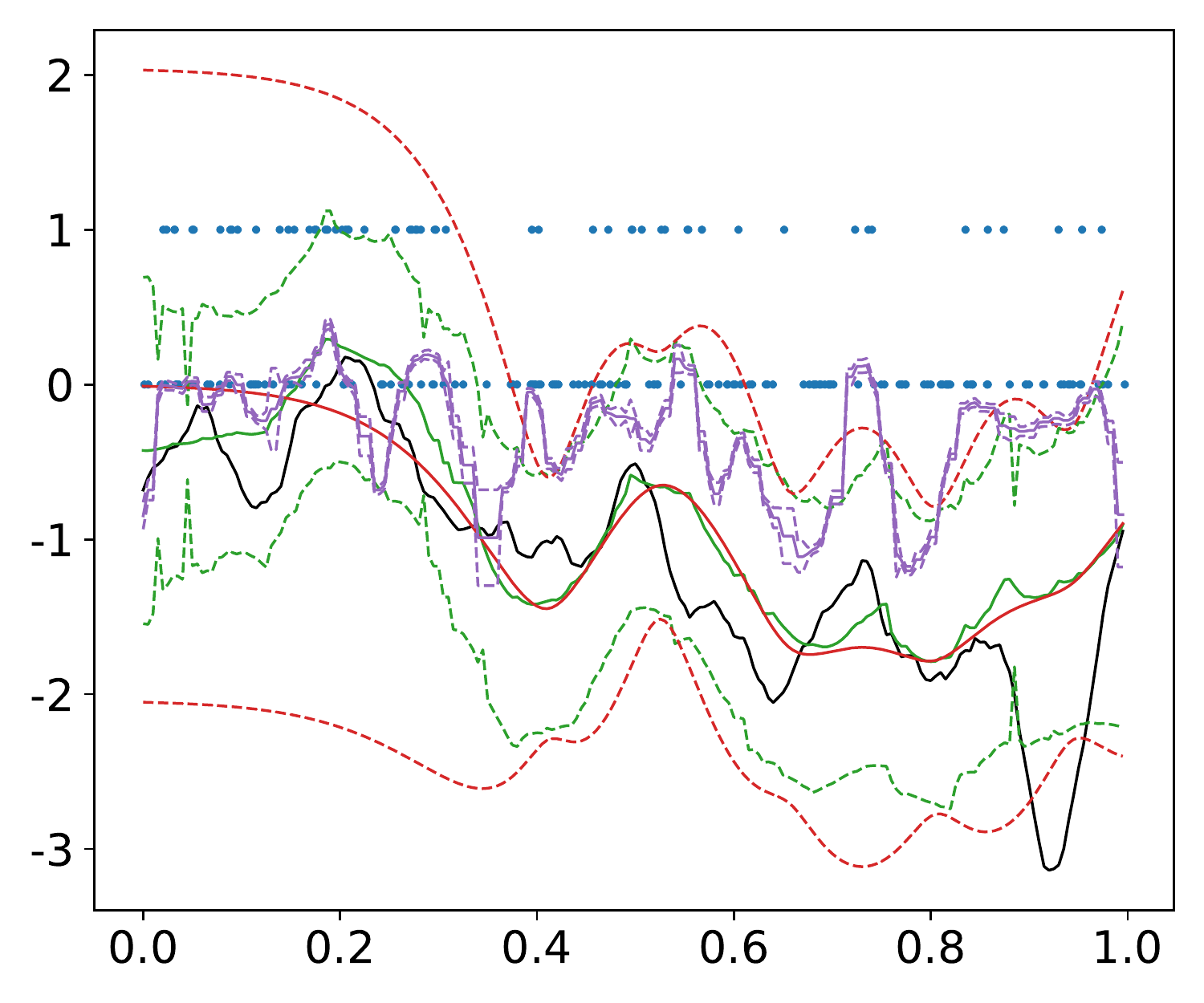}
	\end{subfigure}%
\hfill
   \begin{subfigure}{0.5\textwidth}
	\centering
 	\includegraphics[trim={0 0 3.4in 0},clip, width =.99\linewidth]{plots/1D_legend_combined.pdf}
	\end{subfigure}%
  \caption{Comparison of three variational approximations to the predictive GP posteriors for the Student-$t$ (left) and Bernoulli-logit likelihoods (right). We show the means (solid lines) and $95\%$ pointwise intervals of the posterior predictive distribution $\mathbf{f}^* | \by$ at 200 regularly spaced test inputs. Note that the noise variance $\sigma_{\epsilon} = 0.3$ and range (or length-scale) $\lambda=0.1$ are used. The exact-GP result (DenseGP) is available only for the Gaussian likelihood.}
\label{fig:1D_infer_part2}
\end{figure*}

Similar to Figure~\ref{fig:outsample_f}, Figures~\ref{fig:outsample_f_RBF} and \ref{fig:outsample_f_RQ} provide RMSE and NLL scores at test inputs but for the squared-exponential and rational-quadratic kernels, respectively.
\begin{figure*}[h!]
\centering
	\begin{subfigure}{.99\textwidth}
	\centering
 	\includegraphics[width =.99\linewidth]{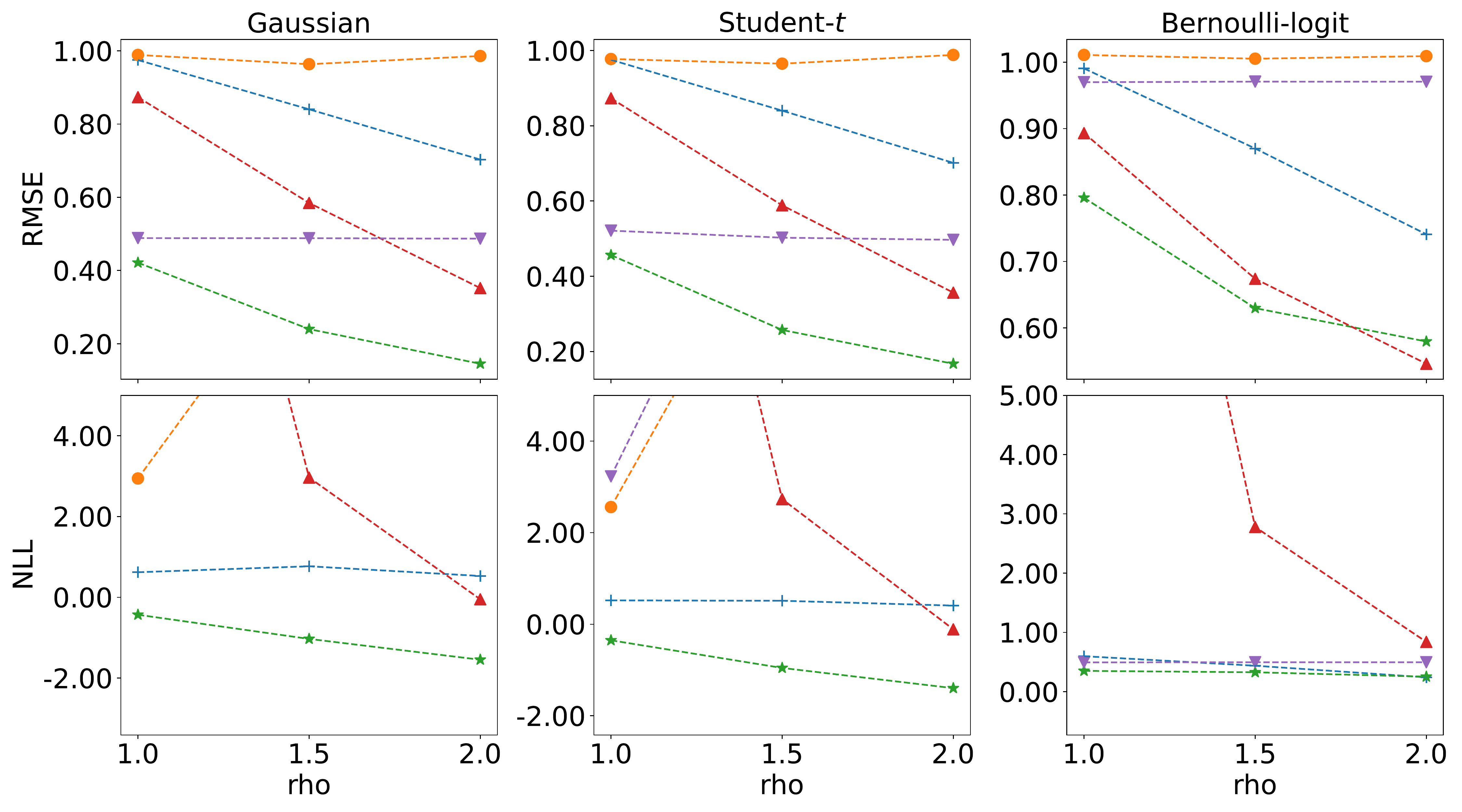}
	\end{subfigure}%

   \begin{subfigure}{0.6\textwidth}
	\centering
 	\includegraphics[width =.99\linewidth]{plots/legend.pdf}
	\end{subfigure}%
	
  \caption{RMSE (top) and NLL (bottom) for predicting the latent field at test inputs for simulated data with the squared exponential kernel in a five-dimensional input domain, as a function of the complexity parameter $\rho$, with Gaussian (left), Student-$t$ (center) and Bernoulli-logit (right) likelihoods. In the bottom panels, some lines are truncated for clearer comparison.}
  \label{fig:outsample_f_RBF}
\end{figure*}

\begin{figure*}[h!]
\centering
	\begin{subfigure}{.99\textwidth}
	\centering
 	\includegraphics[width =.99\linewidth]{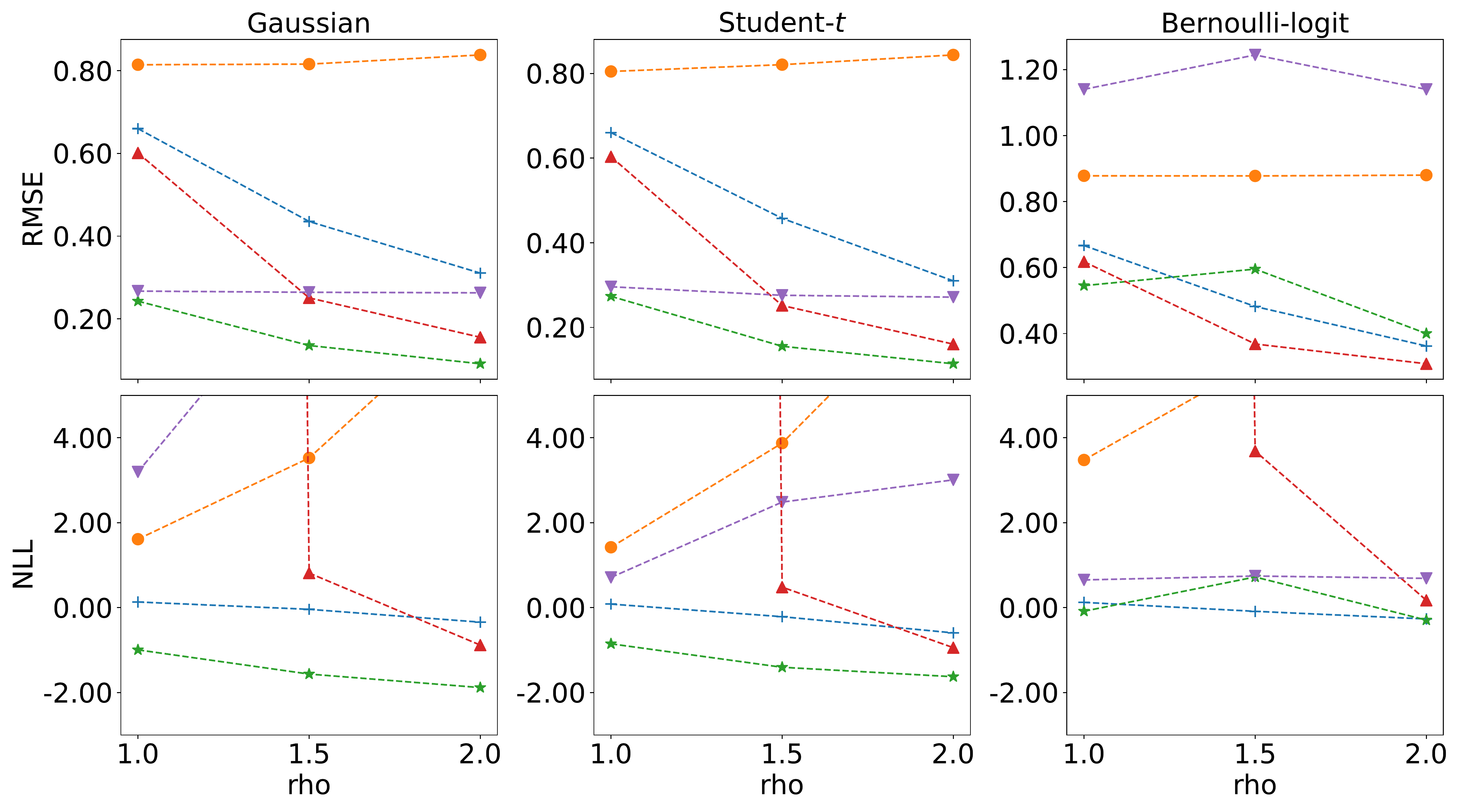}
	\end{subfigure}%

   \begin{subfigure}{0.6\textwidth}
	\centering
 	\includegraphics[width =.99\linewidth]{plots/legend.pdf}
	\end{subfigure}%
	
  \caption{RMSE (top) and NLL (bottom) for predicting the latent field at test inputs for simulated data with the rational quadratic kernel in a five-dimensional input domain, as a function of the complexity parameter $\rho$, with Gaussian (left), Student-$t$ (center) and Bernoulli-logit (right) likelihoods}
  \label{fig:outsample_f_RQ}
\end{figure*}

\section{Graph representation of sparsity patterns and ancestor sets\label{app:graph}}

We here illustrate the sparsity and ancestor sets, using their graph representations. As pointed out by \citet{Katzfuss2017a}, the sparsity patterns can be represented by directed acyclic graphs (DAGs), which also allows straightforward visualization of ancestor sets. Figure~\ref{fig:dag} presents sparsity and ancestor sets for three selected points ($i = 12, 4, 1$) of 16 grid points in the unit square. For example, $\bx_1 = \left(\frac{1}{3} , 1 \right)$ and $\bx_{16} = \left(\frac{2}{3} , \frac{2}{3} \right)$. One can easily see that $\ell_{16} = \infty$, $\ell_{15} = \frac{2 \sqrt{2}}{3}$, $\ell_{14} = \ell_{13} = \sqrt{\left(\frac{1}{3}\right)^2 + \left(\frac{2}{3}\right)^2}$, $\ell_{12} = \ell_{11} = \frac{\sqrt{2}}{3}$ and $\ell_{10} = \cdots = \ell_{1} = \frac{1}{3}$. The edges of the graphs corresponding to the ancestor sets $\anc_{12}$, $\anc_{4}$ and $\anc_{1}$ are denoted by the black curved arrows. Specifically, the sparsity set $\sparsity_1 = \{ 2, 7, 13 \}$, the reduced ancestor set $\tilde\anc_1 = \sparsity_1 \cup \{ 9, 11, 12 \}$ and the (full) ancestor set $\anc_1 = \tilde\anc_1 \cup \{ 15, 16 \}$. Note that $\anc_1$ contains $\tilde\anc_1$, which is a desirable property for leveraging the screening effect in GPs \citep{Stein2011,Bao2020}. This is not always the case for small-scale problems and it depends on distribution of the points, as shown in Figure~\ref{fig:dag2}. Specifically, $\anc_4 = \{ 10, 11, 14, 15, 16 \}$, but $\tilde\anc_4 \setminus \anc_4 = \{ 13 \} \neq \emptyset$. But our numerical studies suggest that $\tilde\anc_i \setminus \anc_i$ are typically empty or very small for large-scale problems, for which computational issues are severe and hence our method is most likely to be used. For relatively large $i=12$, $\sparsity_{12} = \tilde\anc_{12} = \anc_{12} = \{ 16 \}$. As illustrated here, all the reduced ancestor sets include $\bx_{16}$, since $\ell_{16} = \infty$. Otherwise, unlike $\tilde\anc_{4}$ and $\tilde\anc_{1}$, $\tilde\anc_{12}$ does not include $\bx_{15}$ since $\text{dist}(\bx_{15} , \bx_{12}) = \sqrt{2}$ is larger than $\rho \ell_{15} \simeq 1.226$.

\begin{figure}[htbp]
\centering
	\begin{subfigure}{.33\textwidth}
	\centering
 	\includegraphics[width =.99\linewidth]{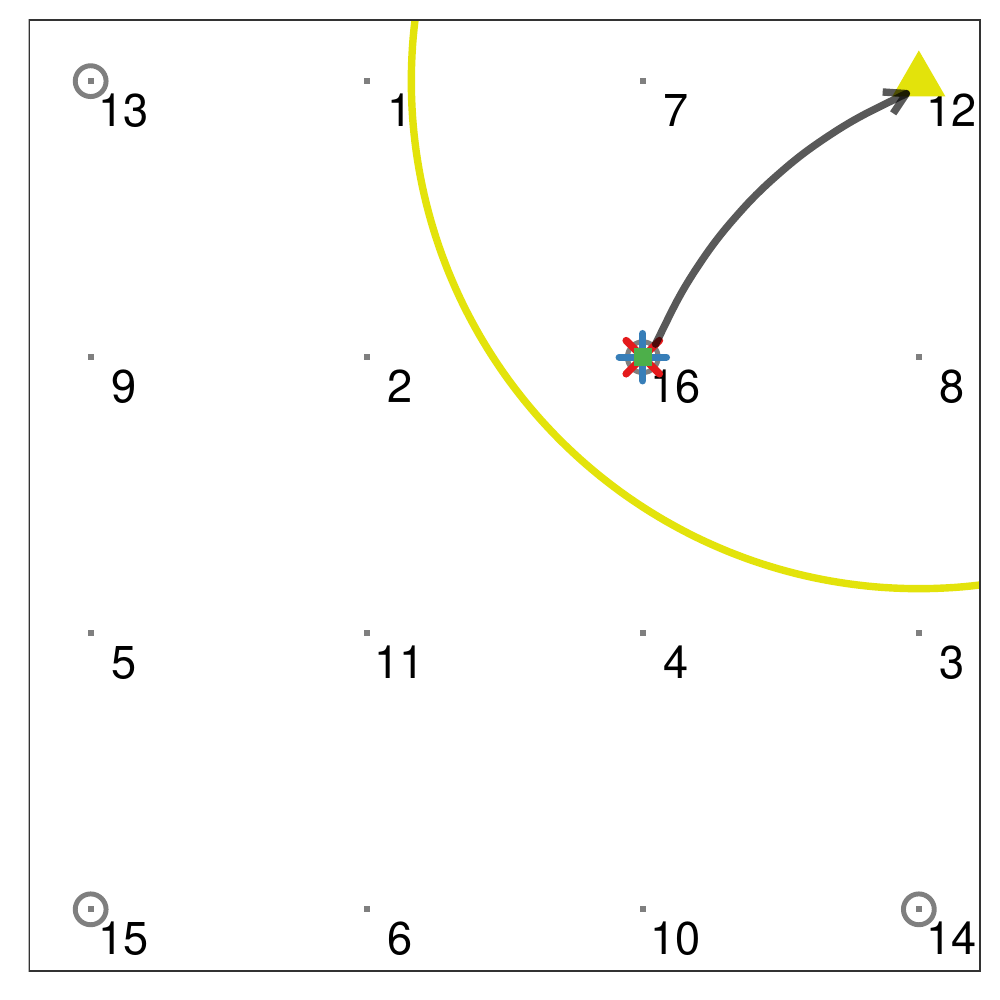}
	\caption{$i=12$}
	\label{fig:dag1}
	\end{subfigure}%
\hfill
	\begin{subfigure}{.33\textwidth}
	\centering
 	\includegraphics[width =.99\linewidth]{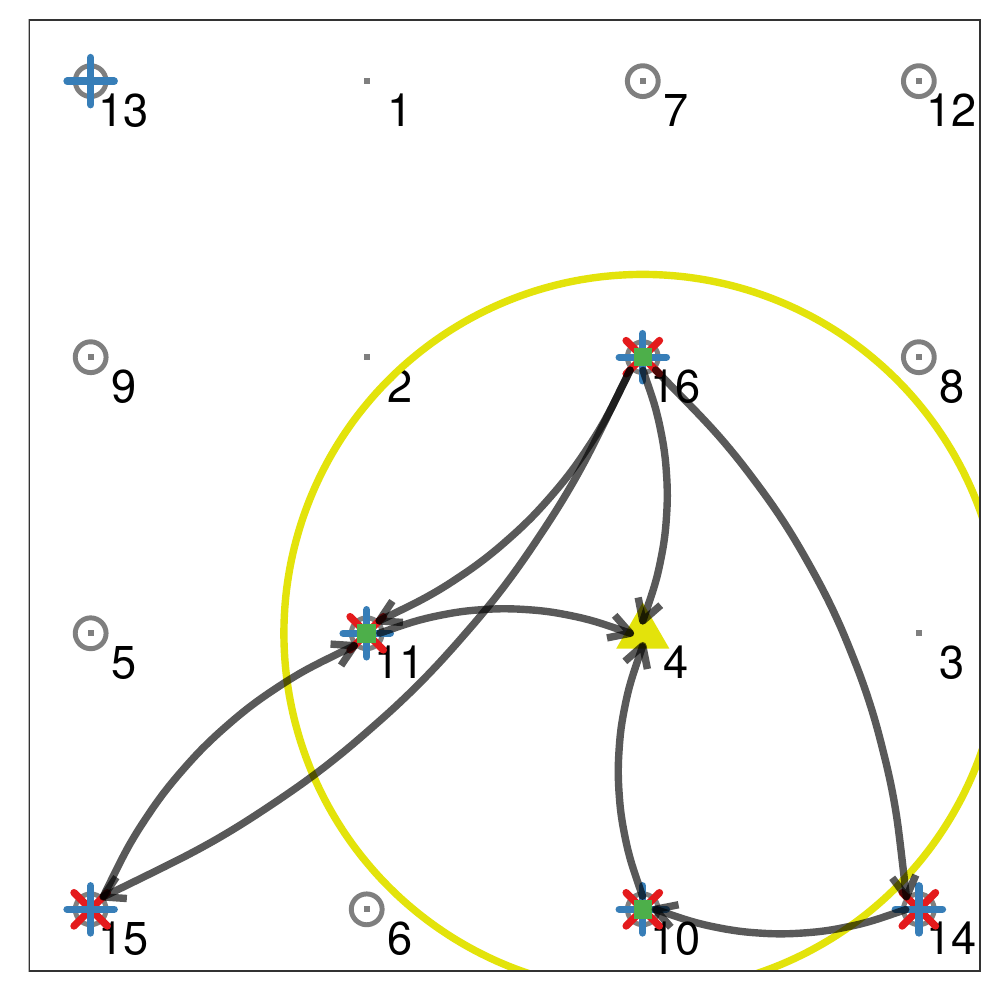}
	\caption{$i=4$}
	\label{fig:dag2}
	\end{subfigure}%
\hfill
	\begin{subfigure}{.33\textwidth}
	\centering
 	\includegraphics[width =.99\linewidth]{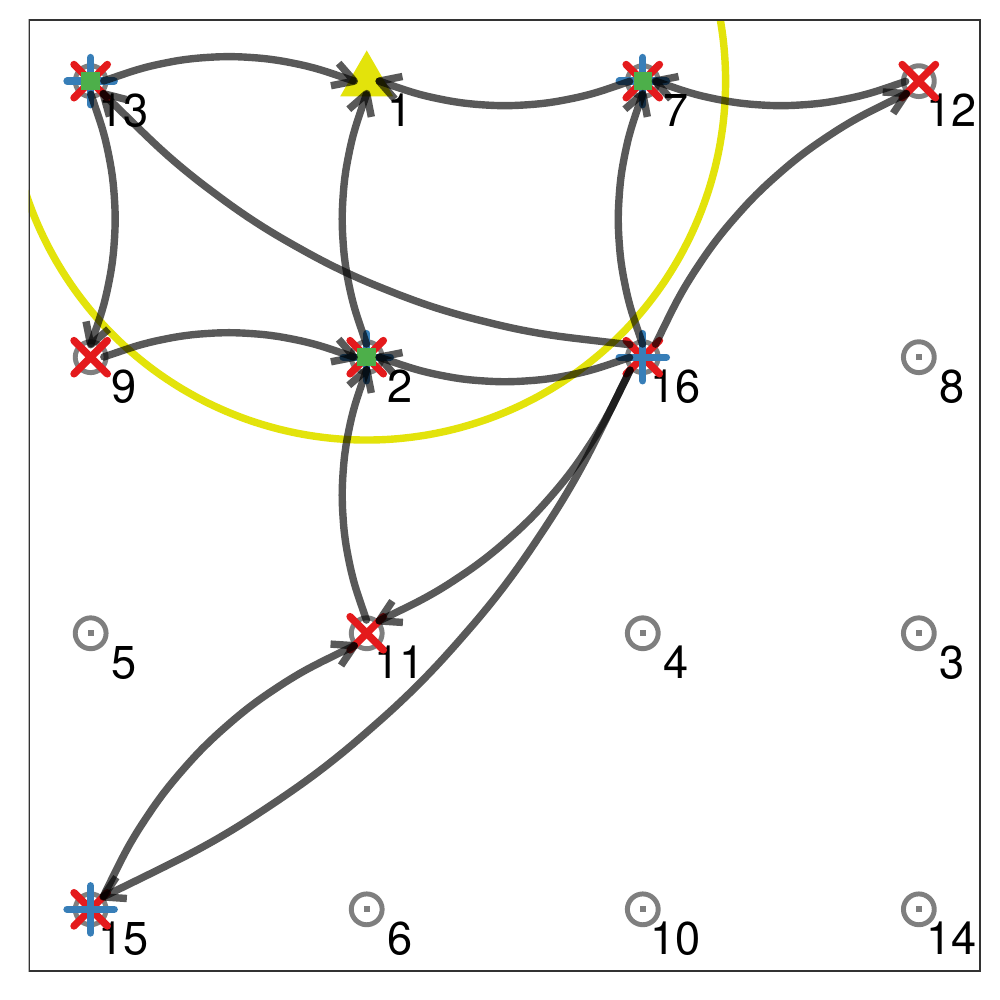}
	\caption{$i=1$}
	\label{fig:dag3}
	\end{subfigure}%
\caption{Reverse maximin ordering on a grid (small gray points) of size $n=4 \times 4 = 16$ on a unit square, $[0,1]^d$ with $d=2$. The $i$th ordered input (${\color[HTML]{e3e30b}\blacktriangle}$), the subsequently ordered $n-i$ inputs (${\color{black}\bm{\medcirc}}$), the distance $\ell_i$ to the nearest neighbor (${\color[HTML]{e3e30b}\boldsymbol{-}}$), the neighboring subsequent inputs $\sparsity_i$ (${\color[HTML]{4DAF4A}\blacksquare}$) within a (yellow) circle of radius $\rho \ell_i$, with $\rho=1.3$, the reduced ancestors $\tilde{\anc}_i$ (${\color[HTML]{377EB8}\bm{+}}$), and the ancestors $\anc_i$ (${\color[HTML]{E41A1C}\bm{\times}}$). The directed acyclic graphs of the sparsity patterns are denoted by arrows ($\bm{\curvearrowright}$).}
\label{fig:dag}
\end{figure}

\section{Proofs \label{app:proofs}}

This section contains the postponed proofs of technical statements in the main paper. A non-rigorous justification for Claim~\ref{claim:reducedancestor} can be also found here.

\begin{proof}[Proof of Proposition~\ref{prop:elbo}]
We have $$\ELBO(q) = \bbE_{q} \log p(\by|\mathbf{f}) - \KL(q(\mathbf{f})\|p(\mathbf{f})),$$
where $\bbE_{q} \log p(\by|\mathbf{f}) = \sum_{i=1}^n \bbE_{q} \log p(y_i|f_i)$. Using a well-known expression for the KL divergence between two Gaussian distributions, we have
\begin{equation}
    2 \KL(q(\mathbf{f})\|p(\mathbf{f})) 
    = \tr\big((\bL\bL^\top)(\bV\bV^\top)^{-1}\big) 
    + (\bfnu-\bfmu)^\top(\bL\bL^\top)(\bfnu-\bfmu) 
    + \log|\bV\bV^\top| - \log|\bL\bL^\top| -n, 
\label{eq:klgaussian}
\end{equation}
where 
$\log|\bV\bV^\top| = 2 \sum_{i=1}^n \log \bV_{ii}$, 
$\log|\bL\bL^\top| = 2 \sum_{i=1}^n \log \bL_{ii}$, 
$(\bfnu-\bfmu)^\top(\bL\bL^\top)(\bfnu-\bfmu) = \sum_{i=1}^n ( (\bfnu-\bfmu)^\top \bL_{:,i} )^2$, 
$\bL_{:,i}$ denotes the $i$th column of $\bL$, and 
\begin{equation}
    \tr\big((\bL\bL^\top)(\bV\bV^\top)^{-1}\big) 
    = \tr\big((\bV^{-1}\bL)^\top(\bV^{-1}\bL)\big) 
    = \sum_{i=1}^n (\bV^{-1}\bL_{:,i})^\top(\bV^{-1}\bL_{:,i}) 
    = \sum_{i=1}^n \|\bV^{-1}\bL_{:,i}\|^2.
\end{equation}
\end{proof}

\begin{proof}[Proof of Proposition~\ref{prop:prior}]
Using a well-known formula for the KL divergence between two Gaussian distributions (e.g., see \eqref{eq:klgaussian}), we have
\begin{equation}
    \KL\big( p(\mathbf{f}) \big\| \tilde{p}(\mathbf{f}) \big) 
    = (\tilde\bfmu - \bfmu)^\top(\tilde\bL\tilde\bL^\top)(\tilde\bfmu - \bfmu)/2 
    + \KL\big( \normal_n(\bfzero,\bK) \big\| \normal_n(\bfzero,(\tilde\bL\tilde\bL^\top)^{-1}) \big),
\end{equation}
which is minimized with respect to $\tilde\bfmu$ by $\tilde\bfmu=\bfmu$, the exact prior mean. Plugging this in, the first summand is zero and the second summand was shown in \citet[][Thm.~2.1]{Schafer2020} to be minimized by an inverse Cholesky factor $\hat\bL$ whose $i$th column can be computed in parallel for $i=1,\ldots,n$ as
\begin{equation}
 \hat\bL_{\sparsity^p_i,i} = \bb_i/\sqrt{\bb_{i,1}}, \quad \text{ with } \bb_i = \bK^{-1}_{\sparsity^p_i,\sparsity^p_i} \be_1.
\end{equation} 
\end{proof}

\begin{proof}[Proof of Proposition~\ref{prop:ancestor}]
$$
\bV^{-1}\hat\bL_{:,i} = \begin{bmatrix} \bV_{1:i-1 , 1:i-1} & \bfzero \\ \bV_{i:n , 1:i-1} & \bV_{i:n , i:n} \end{bmatrix}^{-1} \begin{bmatrix} \bfzero \\ \hat\bL_{i:n , i}\end{bmatrix} = \begin{bmatrix} \bfzero \\ \bV_{i:n , i:n}^{-1} \hat\bL_{i:n , i} \end{bmatrix}    
$$
Let $\bX$ be the inverse of $\bV_{i:n , i:n}$. Then,
$$
(\bV^{-1}\hat\bL_{:,i})_j = \frac{1}{\bV_{j,j}} \left[ \hat\bL_{j , i} - \hat\bL_{j-1 , i} \sum_{r=j-1}^{j-1} \bV_{j,r} \bX_{r-i+1,j-i} - \cdots - \hat\bL_{i , i} \sum_{r=j-1}^{i} \bV_{j,r} \bX_{r-i+1,1} \right]
$$
Since $\sparsity^p_i \subset \anc_i$, $\hat\bL_{j , i} = 0$ for $j \notin \anc_i$. Also, from the definition of $\anc_i$, it can be shown for $j \notin \anc_i$ that 
$
\hat\bL_{j-1 , i} \sum_{r=j-1}^{j-1} \bV_{j,r} \bX_{r-i+1,j-i} = \ldots = \hat\bL_{i , i} \sum_{r=j-1}^{i} \bV_{j,r} \bX_{r-i+1,1} = 0.
$
For instance, suppose $j = i + 1 \notin \anc_i$. Then,
$
(\bV^{-1}\hat\bL_{:,i})_{i+1} = \frac{1}{\bV_{i+1 , i+1}} \left[ \hat\bL_{i+1 , i} - \hat\bL_{i , i} \bV_{i+1 , i} \bX_{1 , 1} \right] = 0,
$
since $\hat\bL_{i+1 , i} = \bV_{i+1 , i} = 0$. Therefore, $(\bV^{-1}\hat\bL_{:,i})_j = 0$ for all $j \notin \anc_i$.
\end{proof}

\begin{proof}[Justification for Claim~\ref{claim:reducedancestor}]

We now provide theoretical justification for our claim that the entries of the vector $\bV^{-1}\hat\bL_{:,i}$ are small outside of $\tilde\anc_i$ with magnitudes that decay exponentially as a function of $\rho$ for each $i = 1 , \ldots , n$. 
In other words, our claim is that for $j \geq i$,
\begin{equation} 
\label{eqn:exponential_decay}
    \log\left(\left| (\bV^{-1}\hat\bL_{:,i} )_{j}\right|\right) \lessapprox \log(n) - \text{dist}\left(\bx_{j}, \bx_{i}\right) / \ell_{j}.
\end{equation}
By the results on exponential screening in \citet{Schafer2017}, the matrix $\hat{\bL}$ satisfies the above decay property for covariances that are Green's functions of elliptic PDEs.
It satisfies even the stronger property with $\ell_{j}$ replaced by $\ell_{i}$.

For a Gaussian likelihood, the matrix $\bV$ satisfies 
\begin{equation}
\label{eqn:ideq_precision}
    \bV \bV^{\top} = \hat{\bL} \hat{\bL}^{\top} + \bR^{-1} =: \bfSigma^{-1}, 
\end{equation}
where $\bR$ is a diagonal covariance matrix of the likelihood.
Interpreted as a PDE, the diagonal matrix $\bR^{-1}$ corresponds to a zero-order term. 
Thus, the associated covariance matrix $(\hat{\bL} \hat{\bL}^{\top})^{-1}$ behaves like a discretized elliptic Green's function and is therefore subject to an exponential screening effect \citep[Section 4.1]{Schafer2020}. 
Let $\bP^{\updownarrow}$ denote the permutation matrix that reverts the order of the degrees of freedom. 
Since $\bP^{\updownarrow} \bV^{-\top} \bP^{\updownarrow}$ is lower triangular and 
\begin{equation}
    \bP^{\updownarrow} \bfSigma \bP^{\updownarrow} = \bP^{\updownarrow} \bV^{-\top} \bP^{\updownarrow} \bP^{\updownarrow} \bV^{-1} \bP^{\updownarrow} = \left(\bP^{\updownarrow} \bV^{-\top} \bP^{\updownarrow}\right) \left(\bP^{\updownarrow} \bV^{-\top} \bP^{\updownarrow}\right)^{\top},
\end{equation}
the matrix $\bP^{\updownarrow} \bV^{-\top} \bP^{\updownarrow}$ is the Cholesky factor of $\bfSigma$ in the maximin (as opposed to the reverse maximin) ordering. 
In \citet{Schafer2017}, it is shown that the Cholesky factors of discretized Green's functions of elliptic PDEs in the maximin ordering have exponentially decaying Cholesky factors. 
In particular, the results of \citet{Schafer2017} suggest that 
\begin{align}
    \forall j \geq i:\  \log\left(\left|\left(\bP^{\updownarrow} \bV^{-\top} \bP^{\updownarrow}\right)_{ji}\right|\right) & \lessapprox \log(n) - \text{dist}\left(\bx_{j}, \bx_{i}\right) / \ell_{i} \\
    \Rightarrow  \forall j \geq i: \  \log\left(\left|\left(\bV^{-1}\right)_{ji}\right|\right) & \lessapprox \log(n) - \text{dist}\left(\bx_{j}, \bx_{i}\right) / \ell_{j}.
\end{align}

As shown, for instance, in \citet[Lemma 5.19]{Schafer2017}, products of matrices that decay rapidly with respect to a distance function $\text{dist}(\cdot, \cdot)$ on its index set, inherit this decay property.
To this end, assume that lower triangular matrices $\bA$ and $\bB$ satisfy this property. 
We then have 
\begin{align}
\log\left(\left|\left(\bA \bB\right)_{ji}\right|\right) &= \log \left(\left|\sum_{k} \bA_{jk} \bB_{ki}\right|\right) 
\leq \log(n) + \log \left(\max_{k} \left| \bA_{jk} \bB_{ki}\right|\right) \\
&\lessapprox \log(n) - \max_{k} \left( \text{dist}\left(\bx_{j}, \bx_{k}\right) / \ell_j - \text{dist}\left(\bx_{j}, \bx_{k}\right) / \ell_k \right). 
\end{align}
By the triangle inequality, we have $\text{dist}\left(\bx_{j}, \bx_{k}\right) + \text{dist}\left(\bx_{k}, \bx_{i}\right) \geq \text{dist}\left(\bx_{j}, \bx_{i}\right)$.
Since the right hand is $- \infty$ unless $j > i$ and thus $\ell_{j} \geq \ell_{i}$, we have thus  
\begin{equation}
    \log\left(\left|\left( \bA \bB\right)_{ji}\right|\right) = \log \left(\left|\sum_{k} \bA_{jk} \bB_{ki}\right|\right) 
    \lessapprox \log(n) - \text{dist}\left(\bx_{j}, \bx_{i}\right) / \ell_{j}, \\
\end{equation}
proving the the result.

For a general exponential family likelihood, the matrix $\bV$ does not necessarily satisfy \eqref{eqn:ideq_precision}. Instead, according to \citet{Nickisch2008}, a quadratic approximation to the log-likelihood under mild conditions implies that 
\begin{equation}
\label{eqn:approx_precision}
    \bV \bV^{\top} = \hat{\bL} \hat{\bL}^{\top} + \bW^{-1}, 
\end{equation}
where $\bW$ is the covariance of the \textit{effective likelihood} obtained by dividing the approximate posterior by the prior. Assuming that $\bW^{-1}$ corresponds to a zero-order term in the context of a PDE, one can also obtain the result from the justification for the Gaussian likelihood case above.
\end{proof}

\begin{proof}[Proof of Proposition~\ref{prop:prediction}]
Note that 
$
p(\mathbf{f}^*|\mathbf{f}) = p(\tilde{\mathbf{f}}) / p(\mathbf{f}) = \normal_{n^*} \left(\bfmu^* + \bK^{*o} \bK^{-1} (\mathbf{f} - \bfmu) , \bK_{*|o}\right), 
$
where $\bK_{*|o} = \bK^{**} - \bK^{*o}\bK^{-1}\bK^{o*}$, and
$
q(\mathbf{f}^*|\mathbf{f}) = q(\tilde{\mathbf{f}}) / q(\mathbf{f}) = \normal_{n^*} \left(\bfnu^* - (\bV^{**})^{-\top} \bV^{o*}{}^{\top} (\mathbf{f} - \bfnu) , (\bV^{**} \bV^{**}{}^{\top})^{-1}\right). 
$
Then, since $\KL \big( p(\mathbf{f}^* | \mathbf{f}) \big\| q(\mathbf{f}^* | \mathbf{f}) \big)$ is a KL divergence between two Gaussian distributions, we have
$$
2 \KL \big( p(\mathbf{f}^* | \mathbf{f}) \big\| q(\mathbf{f}^* | \mathbf{f}) \big) = (\bG \mathbf{f} + \bh)^\top (\bV^{**} \bV^{**}{}^{\top}) (\bG \mathbf{f} + \bh) + 2 \KL \big( \normal_{n^*} \left(\bfzero , \bK_{*|o}\right) \big\| \normal_{n^*} \left(\bfzero , (\bV^{**} \bV^{**}{}^{\top})^{-1}\right) \big)
$$
where $\bG = - (\bV^{**})^{-\top} \bV^{o*}{}^{\top} - \bK^{*o} \bK^{-1}$ and $\bh = \bfnu^* + (\bV^{**})^{-\top} \bV^{o*}{}^{\top} \bfnu - \bfmu^* + \bK^{*o} \bK^{-1} \bfmu$. Using the fact that the first term is quadratic in form, one can show that
$$
\bbE_{p} \Big[ (\bG \mathbf{f} + \bh)^\top (\bV^{**} \bV^{**}{}^{\top}) (\bG \mathbf{f} + \bh) \Big] = (\bG \bfmu + \bh)^\top (\bV^{**} \bV^{**}{}^{\top}) (\bG \bfmu + \bh) + \tr \Big( (\bV^{**} \bV^{**}{}^{\top}) (\bG \bK \bG^\top) \Big).
$$
Then, we can see that $\KL \big( p(\mathbf{f}^* | \mathbf{f}) \big\| q(\mathbf{f}^* | \mathbf{f}) \big)$ is minimized with respect to $\bfnu^*$ by $\bG \bfmu + \bh = \bfzero$. This implies that $\hat{\bfnu}^* = \bfmu^* - (\bV^{**})^{-\top} \bV^{o*}{}^{\top} (\bfnu - \bfmu)$. Plugging this in, we have
\begin{align*}
    \argmin_{\bV^{*} \in \sparsity^*} \bbE_{p} \Big[ \KL \big( p(\mathbf{f}^* | \mathbf{f}) \big\| q(\mathbf{f}^* | \mathbf{f}) \big) \Big] 
    &= \argmin_{\bV^{*} \in \sparsity^*} \Big[ \tr \Big( \bV^{*}{}^{\top} \tilde{\bK} \bV^{*} \Big) - \log \det (\bV^{**} \bV^{**}{}^{\top}) \Big] \\
    &= \argmin_{\bV^{*} \in \sparsity^*} \sum_{i=1}^{n^*} \Big(\bV^*_{\sparsity^*_i , i}{}^\top \tilde{\bK}_{\sparsity^*_i , \sparsity^*_i} \bV^*_{\sparsity^*_i , i} - 2 \log \bV^*_{i,i}\Big)
\end{align*}
Taking the first derivative of the summation with respect to the column vector $\bV^*_{\sparsity^*_i , i}$ and setting it to zero, one can show that $\hat{\bV}^*_{\sparsity^*_i , i} = \tilde{\bK}^{-1}_{\sparsity^*_i , \sparsity^*_i} \be_1 / \bV^*_{i,i}$. Since $\bV^*_{i,i}$ is the first entry of $\hat{\bV}^*_{\sparsity^*_i , i}$, we can have $\hat{\bV}^*_{\sparsity^*_i,i} = \bc_i/\sqrt{\bc_{i,1}}$ where $\bc_i = \tilde{\bK}^{-1}_{\sparsity^*_i,\sparsity^*_i} \be_1$. From the definition of $\sparsity^*_i$, it can be easily shown that $\tilde{\bK}^{-1}_{\sparsity^*_i,\sparsity^*_i} = K(\sparsity^*_i , \sparsity^*_i)^{-1}$.
\end{proof}

\end{document}